\documentclass{article}
\usepackage[main, final, nonatbib]{neurips_2026}
\usepackage[sort,comma,numbers]{natbib}
\usepackage{array}
\usepackage{multirow}
\usepackage[utf8]{inputenc} % allow utf-8 input
\usepackage{subcaption}
\usepackage[T1]{fontenc}    % use 8-bit T1 fonts
\usepackage{hyperref}       % hyperlinks
\usepackage{url}            % simple URL typesetting
\usepackage{booktabs}       % professional-quality tables
\usepackage{amsfonts}       % blackboard math symbols
\usepackage{nicefrac}       % compact symbols for 1/2, etc.
\usepackage{microtype}      % microtypography
\usepackage{xcolor}         % colors
\usepackage{amsmath, amssymb, amsthm}
\usepackage{bm, bbm}
\usepackage{comment}
\usepackage{thm-restate}
\usepackage{enumitem} % lists, etc.
\usepackage{mathtools}
\usepackage[textwidth=2.0cm, textsize=scriptsize]{todonotes}
\usepackage[ruled]{algorithm2e}
\usepackage{algorithmic}
\usepackage{accents}
\usepackage{preamble}
\usepackage[italicdiff]{physics}
\usepackage{pifont}
\newcommand{\cmark}{\ding{51}}   % ✓
\newcommand{\xmark}{\ding{55}}   % ✗
\title{
A Comprehensive View of Fairness \\through Distributional Stability}
\author{%
  Gayane Taturyan \\
  LTCI, Télécom Paris\\Institut Polytechnique de Paris, France\\
  \texttt{gayane.taturyan@telecom-paris.fr} \\
  \And
  Charlotte Laclau \\
  LTCI, Télécom Paris\\ Institut Polytechnique de Paris, France\\
  \texttt{charlotte.laclau@telecom-paris.fr} \\
  \And
  Stephan Clémençon \\
  LTCI, Télécom Paris\\ Institut Polytechnique de Paris, France\\
  \texttt{stephan.clemencon@telecom-paris.fr} \\
}
\begin{document}
% title
\maketitle
% abstract
\begin{abstract}
%Since a predictive rule can be considered unfair in various ways, numerous \textit{ad hoc} concepts have been developed to measure fairness properties. This article is not intended to introduce an additional metric to account for another type of unfairness, but rather to develop a new formal framework that unifies all the concepts proposed to date. It is based on a new application of the concept of stability, which is already widely used in learning theory to quantify an algorithm’s robustness to changes in the empirical training distribution (i.e., in the training data). 

We view fairness as a property of distributional stability. Rather than assessing a predictor under a fixed data distribution, we study how its predictions change under perturbations that modify the composition of protected groups. A predictor is fair if it remains stable under such shifts. Under this perspective, several classical notions of fairness arise as stability with respect to specific perturbations, with the associated unfairness gap given by a Lipschitz constant of a prediction-rate functional. This formulation also yields guarantees that hold uniformly over a range of demographic compositions at test time, without requiring knowledge of the deployment distribution. It leads to a learning procedure based on convex combinations of reweighted predictors, formulated as a second-order cone program, for which we establish generalization bounds. Experiments on standard benchmarks validate our theoretical findings, demonstrate the competitiveness of our approach in terms of fairness, and highlight its distinctive deployment guarantees relative to existing baselines.
\end{abstract}
% ------------ main body ----------
%\input{version-1/notes}
\section{Introduction}
Fairness in the field of machine learning has been studied primarily from two distinct perspectives, leading to the introduction of two broad categories of definitions. \textit{Group fairness}, on the one hand, requires that predictive outcomes be statistically equal or comparable across groups defined by sensitive characteristics such as gender or ethnic origin, which has given rise to a large number of competing metrics, including demographic parity, equal opportunity, or equalized odds, among others \citep{calders2009, hardt2016equality}. On the other hand, the concept of \textit{individual fairness} stipulates as a mandatory condition that similar individuals be assigned similar predictive outcomes, which is formalized by a Lipschitz condition with respect to a task-specific metric \citep{dwork2012fairness}. These two approaches rest on different foundations and have been studied largely in isolation.%, and neither provides a unifying principle.
%Fairness in machine learning has been studied primarily through two families of definitions. 
%Group fairness requires that predictions be statistically equal across groups defined by sensitive attributes such as gender or ethnicity, and has given rise to a large number of competing metrics, including demographic parity, equal opportunity, or equalized odds, among others \citep{calders2009, hardt2016equality}. 
%Individual fairness requires that similar individuals receive similar outcomes, formalized as a Lipschitz condition with respect to task-specific metric \citep{dwork2012fairness}. Both approaches rest on different foundations and have largely been studied in isolation, and neither provides a unifying principle.

We also note that both of these approaches evaluate fairness relative to a fixed data distribution. In practice, a model trained on one population is often deployed on another, whether in different regions, at different times, or in different institutional contexts. Recent work has begun to investigate how fairness guarantees degrade under such changes \citep{giguere2022fairness, jiang2023-chasing, barrainkua}, but these approaches are either tailored to specific shift models, require knowledge of the deployment distribution at training time, or model the shift as acting on the 
feature space while leaving the distribution of the sensitive attribute $G$ itself 
unchanged. Our framework takes a different perspective: we consider perturbations that 
act directly on the marginal distribution of $G$, which is the primary driver of deployment bias in practice.

%Both also evaluate fairness with respect to a fixed data distribution. In practice, a model trained on one population is often deployed on another, either across different regions, time periods, or institutional contexts. Recent work has begun to study how fairness guarantees degrade under such shifts \citep{giguere2022fairness, jiang2023-chasing, barrainkua}. These approaches are either tailored to specific shift models or require knowledge of the deployment distribution at training time.

In this article, we propose to approach fairness from the perspective of stability --- specifically, the distributional stability of a given prediction function rather than the algorithmic stability derived from learning theory. A classifier is fair if its behavior does not change in the face of perturbations in the data distribution that reflect social considerations, such as changes in the composition of protected groups. This single principle links the concepts of fairness at the individual level and at the group level, which appear as special cases corresponding to different classes of perturbations. This naturally resolves the issue of context shifts during deployment: a classifier trained to remain stable in the face of variations in group distribution inherits fairness guarantees that apply uniformly to all demographic compositions at deployment time, without the need to know the deployment distribution in advance.
%We propose to view fairness through the lens of stability, by which we mean the distributional stability of a fixed prediction function rather than the algorithmic stability of learning theory. A classifier is fair if its behavior does not change under perturbations of the data distribution that encode social considerations, such as changes in protected groups composition. This single principle connects individual and group notions of fairness, which emerge as special cases corresponding to different classes of perturbations. It also naturally addresses deployment shift: a classifier trained to be stable under changes in group proportions inherits fairness guarantees that hold uniformly across all possible demographic compositions at deployment time, without requiring knowledge of the deployment distribution in advance.

\textbf{Contributions.} We define fairness of a predictive rule as a stability property (Section~\ref{sec:stability}). We then establish formal equivalences between individual fairness, demographic parity, equal opportunity, and stability under specific perturbation classes, showing that the unfairness gap is in each case a  Lipschitz constant (Section~\ref{sec:equiv}).
For group fairness, we unify these results into a single optimization problem via a constraint functional measuring deviation between predictions and a target level across deployment contexts; 
its decomposition ensures guarantees for both fairness and uniform deployment.
% its decomposition separates sensitivity to demographic shift from the absolute prediction level, thereby providing a guarantee of uniform deployment. 
Restricting to convex mixtures of reweighted classifiers turns the empirical problem into a second-order cone program, for which we establish generalization bounds (Section~\ref{sec:methodology}). Numerical experiments confirm the theory's practical applicability (Section~\ref{sec:exp}). Proofs, the Equalized Odds extension, and additional experiments and discussions are deferred to the Supplementary Material.

\textbf{Related works.}\label{sec:related}
Algorithmic fairness, especially for binary classification, has been extensively studied from both theoretical and methodological perspectives. A large part of this literature considers \textit{aware} setting, in which the sensitive attribute is available at the prediction time~\citep{calders2009,denis2024, zeng2022fair, pmlr-v206-gaucher23a, xian2023fair, Gordaliza19, chiappa2020general}, whereas fewer works address the \textit{unaware} setting~\citep{pmlr-v80-agarwal18a, oneto2020general, chzhen2025randomizedmulticlassclassificationconstraints, hou2024finite,celis2019classification, leteno2025}, in which it is not. A number of works aim to unify fairness-constrained classification within a common framework. A classical example is the reductions approach of \citet{pmlr-v80-agarwal18a}, where the authors reduce fair classification to a sequence of cost-sensitive classification problems and construct a randomized classifier with minimum empirical error under the prescribed constraint. A recent work of \citet{chzhen2025randomizedmulticlassclassificationconstraints} proposes a unified post-processing framework for randomized classification under general system-level constraints (including group-fairness constraints), based on entropic regularization and stochastic optimization. The stability-based perspective is related to recent work on the robustness of fairness-constrained learning under adversarial perturbations \citep{blum24a}, that shows how malicious noise in the training data can significantly affect the validity of fairness constraints. In our work we define fairness itself as a stability under perturbations. Despite an initial resemblance to distributionally robust optimization approaches, our framework is based on specific perturbations induced by a prescribed perturbation rule and admits an importance-weight representation.  Unlike DRO, this structural specificity yields non-vacuous deployment guarantees, as we discuss in Appendix~\ref{app:dro}.
% The stability-based perspective is related to recent work on the robustness of fairness-constrained learning under adversarial perturbations \citep{blum24a}. Their analysis shows how malicious noise in the training data can significantly affect the validity of fairness constraints such as DP or EO. In our case, this could correspond to choosing the perturbation class $\mathcal{P}$ to model adversarial corruptions. The key difference is that we define fairness directly as stability under perturbations, rather than as the resilience of pre-specified fairness constraints to noise. TBD
% +fair classification (notably Agarwal and you). 
% +Discuss the relationship to DRO (the perturbation variance function looks like an average-case relaxation of a worst-case DRO objective) so maybe here and latter in the paper. 

\textbf{Notation.} For an integer \(K>0\), we define \([K]\eqdef \{1,\cdots,K\}\) and write $\Delta_K$ for the probability simplex in $\bbR^K$.
%, and denote by \(\Delta_K \eqdef \{\balpha \in [0,1]^K: \sum_{k\in[K]}\alpha_k = 1\}\) the probability simplex on \(\bbR^K\).
The notation \(\tilde\bigO\) hides unimportant constants and logarithmic terms.
\section{Defining fairness as a property of stability}\label{sec:stability} 
%The concept of stability is ubiquitous in the Sciences and particularly in Mathematics, in that it allows us to describe the properties of a ``system'' (e.g., the solution to a differential equation, the output of a numerical method, or a machine-learning algorithm) when we modify in a certain way the parameters or data that were used to define it. 
The concept of stability is ubiquitous in the sciences, particularly in mathematics, as it describes how a ``system'', such as a machine-learning algorithm, responds to changes in the parameters or data used to define it.
The various notions of stability and their applications are far too numerous to list here and we refer to Appendix~\ref{app:stability} for 
an in-depth description in the machine-learning context. Instead of the main concepts of a predictor’s fairness 
%that have been 
introduced in the literature, 
%which we also review in this section for the sake of clarity, 
we propose 
%to substitute 
a definition based on 
the concept of 
stability under changes in the data distribution. As we will see in the following sections, this new definition allows us to subsume the previous ones within a unified formal framework.
%that is effective for the practical evaluation of fairness properties.
\subsection{Background and preliminaries}
Let us consider the generic problem of predictive learning with fairness concerns regarding sensitive attributes $G$ 
%taking its values in the space 
valued in $\G$, where the random label 
%to be predicted is denoted by 
$Y$
%, with 
takes values in a space $\Y$ equipped with a metric $d_{\Y}:\Y^2\to [0,\infty)$, and the 
%information useful for predicting 
predictive information is modelled by a random vector $\bX$, valued in a high-dimensional metric space $(\X,d_{\X})$. By $\bz=(\bx,g,y)$ is meant a realization in the product space $\mathcal{Z}=\X\times\G\times\Y$ of the random triplet $\bZ = (\bX,G,Y)$, defined on a probability space $(\Omega, \mathcal{F}, \mathbb{P})$ with joint distribution $\cD$ describing the system under study, by $h:\X\to\Y$ an arbitrary (measurable) predictive rule in \textit{unaware} setting. 
%We work in the \textit{unaware} setting, which means that sensitive attributes are assumed to be unavailable at the time of prediction. 
%Let $\X$ be the feature space, $\Y$ the label space, and $\G$ the space of sensitive attributes. We write $z=(\bx,g,y)\in \X\times \G\times \Y$  for a data point, with random variables $(\bX,G,Y)$ jointly distributed according to $\cD$. We consider a measurable prediction function $h:\X\to\{0,1\}$. We work in the \textit{unaware} setting, which means that sensitive attributes are assumed to be unavailable at the time of prediction.
Equipped with these notations, let us now review the standard definitions that aim to formalize 
%the concept of 
algorithmic fairness. Individual fairness (IF), introduced by \citet{dwork2012fairness}, which requires that similar individuals receive similar predictions, is defined as follows.
\begin{definition}[Individual Fairness]
\label{def:IF}
    A predictor $h:\X\to\Y$ satisfies the IF property when it satisfies the $(d_{\Y},d_{\X})$-Lipschitz condition for some constant $L<\infty$:
    \[
  \forall (\bx,\bx')\in \X^2,\;\;  d_{\Y}(h(\bx), h(\bx')) \leq L \cdot d_{\X}(\bx,\bx') \enspace.
    \]
\end{definition}
The concept of group fairness differs in that it requires measures of performance to be identical across groups defined by the sensitive attribute $G$. Depending on the type of predictive problem considered (e.g., classification, regression, ranking) and the criterion chosen to measure performance, it can take various forms.
Among those introduced in the literature, below we highlight the following two major notions for binary classification (\ie $\Y=\{0,1\}$),
%(i.e., when the label $Y$ is binary, valued in $\{0,1\}$), 
that will serve as main examples in the subsequent analysis. For simplicity, the definitions are given for only two sensitive groups, \ie $\G=\{0,1\}$. 
They can be immediately extended to the case where the sensitive attribute $G$ takes more than two values.
%Among the common concepts, we have selected the following. 
\begin{definition}[Demographic Parity]
A predictor $h:\X\to\{0,1\}$ satisfies \emph{Demographic Parity (DP)}
if its predictions are independent of $G$:
$\Pr(h(\bX)=1 \mid G=0) \;=\; \Pr(h(\bX)=1 \mid G=1)$.
\end{definition}
The concept of equal opportunity (EO) requires that there be no disparity with respect to a specific aspect of predictive performance.
\begin{definition}[Equal Opportunity]
A classifier $h$ satisfies the EO property if it achieves
equal true positive rates across groups:
$\Pr(h(\bX)=1 \mid Y=1, G=0) \;=\; \Pr(h(\bX)=1 \mid Y=1, G=1)$.
\end{definition}
%More generally, if $\G$ contains more than two values,
%DP requires $\Pr(h(\bX)=1 \mid G=g)$ to be equal across all $g\in\G$.
%For a multi-valued sensitive attribute $G$, EO requires equality of
%$\Pr(h(\bX)=1 \mid Y=1, G=g)$ across all $g\in\G$.
The concept of Equalized Odds (EOdd) strengthens the EO property by requiring the equality to hold conditionally on both values of the label; its definition is recalled in Appendix~\ref{app:eodds}.
\subsection{A new definition of fairness for a predictive rule}
We now propose to define the fairness of a predictor as a \textit{stability} property. The concept of stability is already widely used in statistical learning theory, primarily to account for
%Stability is a classical concept in statistical learning theory. In its standard form, it concerns 
the sensitivity of a learning algorithm to changes in the training examples \citep{10.1162/153244302760200704, 10.1162/089976699300016304}. An algorithm is said to be stable if replacing a few training data points does not significantly change the learned function. %Refer to the Supplementary Material for more details. 
The framework we are developing here is of a different nature. We consider a fixed prediction function $h$ and seek to quantify how its outputs change as the distribution of sensitive attributes $G$ varies. Unlike \textit{domain adaptation} \citep{redko2022surveydomainadaptationtheory}, which aims to correct such changes, we characterize predictors whose behavior remains stable in the face of these perturbations, and show that classical notions of fairness naturally follow from the requirement for this specific type of stability. 
%Unlike domain adaptation, which seeks to correct for such changes, we characterize predictors whose behaviors is stable with respect to such changes, and show that classical fairness notions arise naturally as instances of the stability requirement. 
% For a  perturbation parameter $\pi$, we denote by $\cD_\pi$ the perturbed distribution, by $\Prob$ the induced probability measure, and by $\Exp$ the corresponding expectation. For the perturbed distribution $\cD_\pi$, we denote the induced probability measure by $\Prob_\pi$, and the corresponding expectation by $\Exp_\pi$. The unperturbed counterparts are written $\Prob$ and $\Exp$. With this in mind, we propose the following definition.
Consider a class $\mathcal{P}$ of (perturbation) operators $\bP:D\subset \mathcal{M}_1(\mathcal{Z})\to \mathcal{M}_1(\mathcal{Z})$ acting on the elements $\mathcal{D}$ of a domain $\mathcal{M}$ included in the set $\mathcal{M}_1(\mathcal{Z})$ of probability measures on $\Z$ to produce a  perturbed probability measure $\bP(\mathcal{D})$ on $\Z$. Since only the values taken by the variable $h(X)$ are relevant to the definition of individual fairness, we set $\Z$ equal to $\X$ in this case.
%For a perturbation parameter $\pi$, we denote by $\cD_\pi$ the perturbed distribution. We denote the induced probability measure by $\Prob_{\cD_\pi}$, and the corresponding expectation by $\Exp_{\cD_\pi}$. The unperturbed counterparts are written $\Prob$ and $\Exp$. 
With this in mind, we propose the following definition, inspired by the formalism of PAC learning.

\begin{definition}[Fairness as Distributional Stability]
\label{def:stability}
Consider a class $\cP$ of perturbation operators acting on a domain $D\subset \mathcal{M}_1(\mathcal{Z})$ and a prediction function $h:\X \to \Y$. Let $\varepsilon>0$ and $\Delta(h;\; (\mathcal{D},\bP))$ be a statistical distance between the distributions of the prediction $h(X)$ under $\mathcal{D}\in \mathcal{M}$ and under $\bP(\mathcal{D})$. The prediction function $h$ is said to be
$\varepsilon$--stable w.r.t. $\cP$ and $\Delta$ if and only if   
\[
 \sup_{(\mathcal{D},\bP)\in \mathcal{M}\times \mathcal{P}}\Delta(h;\; (\mathcal{D},\bP)) \leq \varepsilon \enspace.
\]
%where $\Delta(h; \pi)$ measures the discrepancy between the predictions of $h$ under $\cD$ and under $\cD_\pi$, such as differences in the prediction rate or the performance at the group-level.
\end{definition}
The choice of $\cP$ and that of $\Delta$ are left intentionally general here; specific and interpretable instances are constructed in Section~\ref{sec:equiv}.
As we shall see, it is natural to consider perturbations of the distribution of the sensitive attribute $G$ in the context of group fairness. 
% This connects our framework to the literature on distributionally robust optimization, though our approach does not involve worst-case optimization over a predefined uncertainty set.
% we discuss this connection further in Section~\ref{sec:related}.
%This definition formalizes the intuition that fairness requires
%\emph{insensitivity to relevant changes}. Unlike group or individual fairness, which impose parity constraints with respect to pre-specified attributes or metrics, this approach defines fairness as the stability of predictions and performance under perturbations that encode social considerations.
%\begin{remark}
%Our notion of stability doesn't fall under classical algorithmic-stability definitions, since it doesn't involve retraining the learning algorithm on perturbed datasets. Instead, it considers a \textit{fixed} prediction function, and measures how its predictions change under data perturbations according to some measure. Thus, our notion of stability refers to distributional stability, connecting this also to the literature of domain adaptation (and distributional robustness).
%\end{remark}
\section{From various existing concepts of fairness to a unified approach}\label{sec:equiv}

In this section, we show that the general formalism introduced accounts for the various concepts of individual and group fairness
for specific choices of the class of $\cP$ and the statistical distance $\Delta$.
\subsection{Individual fairness and local perturbations}
The feature space $\X$ being equipped with a task-specific metric $d_{\X}$, as in \citet{dwork2012fairness}, consider the class of transformations of $\X$ indexed by $\rho>0$:
%, we consider perturbations that move each data point $\bx \in \X$ to a nearby point $T(\bx)$ within radius $\rho$, and define the class of admissible transformations as . 
\begin{equation}\label{eq:T}
\cT_\rho = \{T:\X\to\X; \quad T \text{ measurable} \quad \text{and} \quad \forall \bx\in\X,\,\, d_{\X}(\bx, T(\bx))\le \rho \} \enspace.
\end{equation}
Each transformation $T\in \mathcal{T}_{\rho}$ is naturally associated with a perturbation operator $\bP_T$ acting on the distribution of $X$ and assigning to any probability distribution $\mathcal{D}$ on $\X$ its pushforward by $T$.
\begin{proposition}[IF and Stability]
    \label{prop:indiv-worst-stab}
    Let $h:\X\to\Y$ be a prediction function 
    and $d_{\Y}:\Y^2\to [0,\infty)$ a metric on predictions. The following assertions are equivalent.
    %\begin{enumerate}
        %\item 
        
    \,1. The predictive function $h$ satisfies the IF property with $(d_{\Y},d_{\X})$-Lipschitz constant $L<\infty$.
        %\item 
    
    \,2. For all $\rho >0$, the predictive function $h$ is $L\rho$--stable  w.r.t. to the class $\cP_{\rho}$ of perturbation operators $\bP_{T}$, $T\in \mathcal{T}_{\rho}$, defined on the domain $\mathcal{M}=\{\delta_\bx:\; \bx\in \X\}$ of all point masses and taking $\Delta(h;\; (\mathcal{D},\bP))$ as the Wasserstein $1$-distance between the pushforward of $h(\bX)$'s distribution under $\mathcal{D}$ and that under $\bP(\mathcal{D})$  based on the cost function $d_{\Y}$:
        \begin{equation}
        \label{eq:IF_stat_dist}
        \sup_{(\bx,T)\in \X\times \mathcal{T}_{\rho}}d_{\Y}(h(\bx),h(T(\bx)))\leq L\rho \enspace.
        \end{equation}
    %\end{enumerate}
\end{proposition}
% Recall that the Wasserstein $1$-distance between the pushforward of $h(\bX)$'s distribution under $\mathcal{D}$ and that under $\bP(\mathcal{D})$  based on the cost function $d_{\Y}$ is \(\inf_{\bX\sim \mathcal{D},\; \bX'\sim \bP(\mathcal{D})}\mathbb{E}[d_{\Y}(h(\bX),h(\bX')],\) where the infimum is taken over all couplings $(\bX,\bX')$ of the pair of distributions $\mathcal{D}$ and $\bP(\mathcal{D})$.
%In this setting, the perturbation parameter is a transformation $T\in \cT_\rho$ and stability is measured pointwise: a predictor is stable under $\cT_\rho$ if its worst-case prediction change over all individuals and all admissible transformations remains small. 
%\begin{proposition}[IF and Stability]
 %   \label{prop:indiv-worst-stab}
  %  Let $h:\X\to\Y$ be a prediction function, $(\X,d)$ a metric space, 
  %  and $\Delta$ a metric on predictions. The following statements are equivalent.
%    \begin{enumerate}
 %       \item $\delta(h(\bx),h(\bx')) \leq L d(\bx,\bx')  \enspace, \forall \bx,\bx'\in \X$
  %      \item $\sup\limits_{T\in \cT_\rho}\sup\limits_{\bx\in\X} \Delta(h(\bx),h(T(\bx))) \leq L \rho \enspace, \forall \rho > 0 \enspace.$
  %  \end{enumerate}
%\end{proposition}
Individual fairness is therefore equivalent to distributional stability in the sense of Definition \ref{def:stability}, with a stability parameter $L\rho$ proportional to the perturbation parameter $\rho$. 
%This formulation also admits an equivalent distributional form in the language of Definition \ref{def:stability} via pushforward distributions, which we detail in Appendix~\ref{app:proofs:if}. 
\subsection{Group fairness and resampling perturbations}
We now consider the various notions of group fairness. Rather than perturbing individual feature vectors, we consider perturbations that act on the marginal distribution of the sensitive attribute $G$ (assumed binary for simplicity).
%, while leaving the $(\bX,Y)$'s conditional distribution given $G$ unchanged. 
The magnitude of the perturbation is parametrized here by the proportion $\pi\in (0,1)$ of group $G=1$ in the population, and $\Delta$ measures how the prediction rates of $h$ change as this proportion varies. 
We introduce two types of perturbation. The first resamples $G$ globally, while the second resamples $G$ within the positive label class. 
\begin{perturbation}\label{pert:G} 
The class $\cP^{G}_{\pi}$ is composed of all operators $\bP:\mathcal{M}_1(\Z)\to \mathcal{M}_1(\Z)$ such that, for all distribution $\mathcal{D}\in \mathcal{M}_1(\Z)\enspace,$\,
1) the conditional distribution of $(\bX,Y)$ given $G$ is the same under $\mathcal{D}$ and under $\bP(\mathcal{D})\enspace,$ \,
2) $\Prob_{\bP(\cD)}(G=1)=\pi\enspace.$
%\[
%\cD^\cG
%= \left\{ \cD_\pi  \quad\middle|\quad
%\begin{aligned}
%& \pi \in [0,1] \qquad \Prob_{\cD_\pi}(G=1)=\pi \\
%& \Prob(\bX,Y\mid G) = \Prob_{\cD_\pi}(\bX,Y\mid G)
%\end{aligned}
%\right\}\enspace.
%\]
\end{perturbation}
%ù\begin{perturbation}\label{pert:G} 
%The family $\cD^{G}$, is obtained from $\cD$ by resampling $G$ so that $\Prob_{\cD_\pi}(G=1)=\pi$,
%while keeping $\Prob(\bX,Y\mid G=g)$ unchanged. 
%\[
%\cD^\cG
%= \left\{ \cD_\pi  \quad\middle|\quad
%\begin{aligned}
%& \pi \in [0,1] \qquad \Prob_{\cD_\pi}(G=1)=\pi \\
%& \Prob(\bX,Y\mid G) = \Prob_{\cD_\pi}(\bX,Y\mid G)
%\end{aligned}
%\right\}\enspace.
%\]
%\end{perturbation}
Such a class is non empty, since the type 1 perturbations of a distribution $\mathcal{D}$ are in 1-to-1 correspondence with the Bernoulli random variables on $(\Omega, \mathcal{F},\mathbb{P})$ with parameter $\pi$.
%one can construct $\cD_\pi$ by resampling $G$ independently of $(X,Y)$ within each group. 

\begin{perturbation}
\label{pert:G-1} The class $\cP^{G\mid 1}_{\pi}$ is composed of all operators $\bP:\mathcal{M}_1(\Z)\to \mathcal{M}_1(\Z)$ such that, for all distribution $\mathcal{D}\in \mathcal{M}_1(\Z)\enspace,$\, 1) the conditional distribution of $\bX$ given $(Y,G)$ and the distribution of $Y$ are the same under $\mathcal{D}$ and under $\bP(\mathcal{D})\enspace,$ \, 2) $\Prob_{\bP(\cD)}(G=1\mid Y=1)=\pi\enspace.$
\end{perturbation}
% \begin{perturbation}
% \label{pert:G-1} The family
% $\cD^{G\mid1}$ is obtained from $\cD$ by resampling $G$ so that $\Prob_{\cD_\pi}(G=1 \mid Y=1)=\pi$,
% while keeping $\Prob(\bX\mid Y,G)$ and $\Prob(Y)$ unchanged. 
% \end{perturbation}
\begin{remark}[On perturbations] 
In Perturbation~\ref{pert:G-1}, we fix $\Prob(\bX\mid Y,G)$ rather than the stronger $\Prob(\bX, Y\mid G)$ used in Perturbation~\ref{pert:G}. While $\Prob(\bX, Y\mid G)$ implies the invariance of $\Prob(\bX\mid Y,G)$, it also implies the invariance of $\Prob(Y\mid G)$. This, combined with the imposed constraint of $\Prob_{\bP(\cD)}(G=1 \mid Y=1)=\pi$, may be infeasible by Bayes' rule. Fixing only $\Prob(\bX\mid Y,G)$ ensures feasibility. 
\end{remark}
Given a perturbation operator $\bP\in\cP_\pi^{\scriptscriptstyle\square}$, with $\square\in\{G,G\mid 1\}$ and a base distribution $\cD$, we write
$\cD_\pi\eqdef\bP(\cD)$.
We show that group fairness constraints can be expressed by a certain descriptive \textit{functional} that depends on the perturbation setting and is affine \wrt the perturbation parameter. Throughout the paper, we reserve the label \emph{Claim} for statements that 
generalize a pattern observed on the running examples (\eg DP, EO), under conditions tailored to each perturbation rule.

\begin{definition}[Prediction rate functional]\label{def:Phi}
For a prediction function $h:\X\to\{0,1\}$ and a perturbation parameter $\pi \in [0,1]$, we define the functional
\(
 \Phi (h; \pi) \eqdef \Exp_{\cD_\pi}[\phi_h (\bZ)] \enspace, 
\)
where $\phi_h : \Z\to[0,1]$ is some statistic that is linear \wrt $h$.
\end{definition}
\begin{claim}\label{claim:Phi}
Let $\cD_\pi$ be a perturbed distribution obtained from $\cD$ according to a \textit{given, appropriate} perturbation rule and a fixed perturbation parameter $\pi\in[0,1]$. The prediction rate functional $\Phi(h;\pi)$ admits a form
\begin{equation}\label{eq:Phi}
\Phi (h; \pi)  = \alpha (h) + \pi \gamma (h)   \enspace,
\end{equation}
where $\alpha(h)$ and $\gamma(h)$ are some prediction rate quantities under the initial (unperturbed) distribution $\cD$. Moreover, it holds that
\begin{align}
& \label{eq:Phi-lip} | \Phi (h; \pi) - \Phi (h; \pi') | = | \gamma (h)||\pi - \pi'| \enspace, \quad \forall \pi, \pi' \in [0,1] \enspace, \text{ and} \\
& \label{eq:Phi-unf}
 \sup_{\pi,\pi' \in [0,1]} | \Phi (h; \pi) - \Phi (h; \pi') | = | \gamma (h)| \eqdef \unf(h) \enspace,
\end{align}
where $\unf(h)$ is the explicit fairness gap (unfairness) for each constraint.
\end{claim}

The functional $\Phi(h;\pi)$, although at first obscure, clearly captures the intuition behind each constraint, by being defined as a relevant prediction rate under the perturbed distribution. We now show that DP and EO correspond, respectively, to stability under the introduced perturbations. Throughout, $h:\cX\rightarrow\{0,1\}$ is a possibly randomized prediction function and $G\in \{0,1\}$ is the sensitive attribute. 

\textbf{Demographic Parity and Stability.} Let $\cD_\pi$ be a perturbed distribution according to Perturbation~\ref{pert:G}. We define the \textit{global positive} prediction rate under $\cD_\pi$ as $P_h(\pi) \eqdef \Prob_\pi(h(\bX)=1) = \Exp_{\cD_\pi}[\ind{h(\bX)}]$ (or $\Exp_{\cD_\pi}[{h(\bX)}]$ when randomized) and denote by $r_g (h) \eqdef \Prob(h(\bX)=1\mid G=g)$\enspace. 
\begin{proposition}
\label{prop:DP-stab}
The prediction $h$ satisfies DP if and only if $P_h(\pi)$ is constant in $\pi$. Moreover, 
\[
P_h(\pi) = r_0(h) + \pi(r_1(h)-r_0(h))\enspace,\text{ and } \sup_{\pi,\pi'\in[0,1]}|P_h(\pi)-P_h(\pi')| = |r_1(h)-r_0(h)|\enspace.
\]
\end{proposition}
% \begin{proposition}[Demographic Parity and Stability]\label{prop:DP-stab}
% Let $r_g (h) \eqdef \Prob(h(\bX)=1\mid G=g)$ and $P_h(\pi) \eqdef \Prob_{\cD_\pi}(h(\bX)=1)$ for $\cD_\pi\in\cD^G$. Then $h$ satisfies Demographic Parity if and only if $P_h(\pi)$ is constant in $\pi$. Moreover, 
% $$
% P_h(\pi) = r_0(h) + \pi(r_1(h)-r_0(h)), \qquad \sup_{\pi,\pi'\in[0,1]}|P_h(\pi)-P_h(\pi')| = |r_1(h)-r_0(h)|.
% $$
% \end{proposition}
The deviation from stability coincides exactly with the DP gap. This result also connects DP to sample representativity: if the training dataset has group proportions that differ from the population, resampling perturbations model precisely this shift. 
\begin{example} Suppose the model predicts a positive outcome for $80\%$ of individuals in group $0$ and for $40\%$ of individuals in group $1$, so $r_0(h)=0.8$ and $r_1(h)=0.4\enspace.$ The global positive prediction rate under a population with group prevalence $\pi$ is 
$P_h(\pi) = 0.8\,(1-\pi) + 0.4\,\pi = 0.8-0.4\,\pi\enspace.$
%As $\pi$ varies from $0$ to $1$, this rate moves linearly from $0.8$, when the population consists entirely of group $0$, down to $0.4$, when it consists entirely of group $1$. The maximal deviation is
Thus, as $\pi$ varies from $0$ to $1$, the prediction rate decreases linearly from $0.8$ to $0.4$, and the maximal deviation is
$\sup_{\pi, \pi'\in[0,1]}|P_h(\pi)-P_h(\pi')|=|r_0-r_1|=0.4\enspace,$
which is the DP gap.
\end{example} 
The same analysis applies to intra-class resampling. 

\textbf{Equal Opportunity and Stability.} Let $\cD_\pi$ be a perturbed distribution according to Perturbation~\ref{pert:G-1}. We define the \textit{true positive} prediction rate under $\cD_\pi$ as $P_{h\mid Y=1}(\pi) \eqdef \Prob_{\cD_\pi}(h(\bX)=1 \mid Y=1) = \Exp_{\cD_\pi} \big[\tfrac{\ind{h(\bX)}
\ind{Y=1}}{\Prob(Y=1)}\big]$ (or $\Exp_{\cD_\pi} [\tfrac{{h(\bX)}
\ind{Y=1}}{\Prob(Y=1)}]$ when randomized) and denote by $q_{g} (h) \eqdef \Prob(h(\bX)=1\mid G=g, Y=1)$\enspace. 
\begin{proposition}\label{prop:EOpp-stab}
The prediction $h$ satisfies EO if and only if $P_{h|Y=1}(\pi)$ is constant in $\pi$. Moreover, 
\[
P_{h\mid Y=1}(\pi) = q_0(h) + \pi(q_1(h)-q_0(h))\enspace,\text{ and } \sup_{\pi,\pi'\in[0,1]}|P_{h\mid Y=1}(\pi)-P_{h\mid Y=1}(\pi')| = |q_1(h)-q_0(h)|\enspace.
\]
\end{proposition}
\begin{claim}
\label{claim:phi-bd}
Due to linearity of $\phi_h$ \wrt $h$ (Definition~\ref{def:Phi}), there exists a constant $B>0$ such that for every measurable function $f: \X \to \bbR$, it holds that $\Exp[\phi_f(\bZ)]\leq B\Exp|f(\bX)|$. For example, in the case of Demographic Parity, we have $\phi_h(\bX)=h(\bX)$ and $B=1$.
\end{claim}
% These two results share a common structure: in both cases, the relevant prediction rate is linear in $\pi$, and the deviation from stability is exactly the fairness gap. This observation is the starting point for our unified framework developed next. 
\textbf{Variance control.} It is easy to notice that in all of the examples we have $\sup_{\pi,\pi'\in[0,1]} | \Phi (h; \pi) - \Phi (h; \pi') | = | \Phi (h; 1) - \Phi (h; 0) |$, meaning the supremum is achieved at the endpoints. This rewrites the problem as the usual ``minimization under DP/EO''. However, our interest lies within the resampling aspect of the problem. Rather than focusing on the worst-case deviation over all perturbations, we now adopt a stochastic viewpoint, in which perturbations are achieved according to a distribution $\Pi$. This allows us to quantify fairness in terms of average sensitivity, leading to smoother and more tractable constraints. We propose to reformulate the constraint through variance control.

\begin{definition}[Perturbation variance functional]
\label{def:pert-var}
Let ${\Pi}$ be a probability measure on $[0,1]$ and $\pi\sim\Pi$. We define the perturbation variance functional as
\(
\bV_\Pi (h) \eqdef \Var_{\pi\sim\Pi} (\Phi(h; \pi)) \enspace.
\)
\end{definition}
From Definition~\ref{def:Phi}, we have that $\Phi(h; \pi) = \alpha (h) + \pi \gamma (h)$. Using the fact that for $a,b\in\bbR$ and a random variable $X$, $\Var(a+bX) = b^2 \Var(X)$, we get 
\(
\label{eq:variance}
\bV_\Pi(h) = \gamma(h)^2 \Var_{\pi\sim\Pi} (\pi) \enspace.
\)
This equation above leads to the equivalence 
\(
\label{eq:DP-var-equiv}
\enspace \unf(h) \leq \varepsilon \enspace \iff \enspace \bV_\Pi(h) \leq \varepsilon^2 \Var_{\pi\sim\Pi} (\pi) \enspace. 
\)

\textbf{Links between individual and group fairness.} As we established, both individual and group fairness can be equivalently expressed by our notion of stability from Definition~\ref{def:stability}. Let us emphasise one more connection between the two. The functional $\Phi (h; \pi)$ is \textit{Lipschitz} \wrt perturbation parameter $\pi$~\eqref{eq:Phi-lip}, and its Lipschitz constant is exactly the unfairness gap. This provides a complementary angle to view fairness notions: while individual fairness requires similar predictions for similar individuals, formalized by Lipschitz property of the prediction function \wrt the feature space, our generalized group-fairness formulation requires stability of perturbation-indexed prediction rates under small changes in the perturbation parameter, that is, Lipschitz continuity of the map $\pi\to\Phi(h;\pi)$. However, despite these conceptual analogies and the possibility of extending our variance-based perspective to individual fairness through suitably randomized local perturbations, the methodologies required to achieve fairness in the two settings remain structurally different. We therefore focus from now on on group fairness, leaving explorations of corresponding methods for individual fairness to future work.
\section{Methodology: variance and mean control with mixture models}\label{sec:methodology}
\subsection{Problem formulation}
\textbf{Variance and mean control.} Under the introduced perturbation-variance functional constraint, the problem reduces to classical fairness constraints, with a scaling governed by the variance of the perturbation-parameter distribution. This holds even for the empirical problem. Rather than controlling only the variability of $\Phi(h;\pi)$ across perturbations, we consider a stricter objective of requiring $\Phi(h;\pi)$ to remain close to a given target level $t$. 
%To capture this requirement, we introduce the following functional. 
\begin{definition}[Perturbation $t$-variance functional]
\label{def:pert-t-var}
Let $\Pi$ be a probability measure on $(0,1)$ and $\pi\sim\Pi$. For a $t\in[0,1]$, we define the perturbation $t$-variance functional as
\[
    \cC_\Pi (h; t) \eqdef \Exp_{\pi\sim{\Pi}} [(\Phi(h;\pi) - t)^2] \enspace.
\]
\end{definition}
Adding and subtracting $\Exp_{\pi\sim\Pi}[\Phi(h;\pi)]$ inside the square yields the following decomposition:
% , and noticing that $\Exp_{\Pi}[2(\Phi-\Exp_{\Pi}[\Phi])(\Exp_\Pi [\Phi] -t)] = 2(\Exp_\Pi [\Phi] -t)\Exp_{\Pi}[(\Phi-\Exp_{\Pi}[\Phi])] = 0$, we get the following decomposition, 
\begin{equation}
    \label{eq:t-var-decomp}
    \cC_{\Pi}(h;t) = \bV_{\Pi}(h) + (\Exp_{\pi\sim\Pi}[\Phi(h;\pi)]-t)^2 \enspace.
\end{equation}
As the decomposition~\eqref{eq:t-var-decomp} shows, the functional $\cC_{\Pi}(h;t)$ controls not only the variance term equivalent to fairness, but also the mean level of $\Phi(h;\pi)$ itself. 
Moreover, one can notice that
\(
\bV_\Pi(h) = \inf_{t\in[0,1]} \cC_{\Pi}(h;t)\enspace,\text{ with } t^\star = \Exp_{\pi\sim\Pi}[\Phi(h;\pi)] \enspace.
\)
Thus, we can state the following obvious lemma. 
\begin{lemma} If $h$ satisfies
$\cC_{\Pi}(h;t) \leq  \delta$ for $t\in[0,1],$ then $\bV_\Pi (h) \leq \delta$ and $(\Exp_{\pi\sim\Pi}[\Phi(h;\pi)]-t)^2 \leq \delta$.
\end{lemma}
\begin{remark}[On the choice of $t$] 
We can choose $t$ as desired, or, more naturally, we can set $t=\Exp_{\pi\sim\Pi}[\Phi(h_0,\pi)]$, where $h_0$ is an unconstrained baseline. The impact of $t$ is studied in Appendix~\ref{app:t-study}.
\end{remark}
We introduce the following standard regularity assumption on the \textit{loss}. 
\begin{assumption}
\label{ass:loss:bd} The loss $\ell:[0,1]\times\Y\to[0,1]$ is convex and $L$-Lipschitz in its first argument.
\end{assumption}
%Given a \textit{loss function} $\ell:[0,1]\times\cY \to \bbR_+$, 
We define the corresponding population \textit{risk} of a predictor $h:\X\to [0,1]$ by $\risk(h) \eqdef \Exp [\ell(h(\bX), Y)]$. Given a $t\in[0,1]$, our goal is to solve the following constrained problem 
\begin{equation}
\label{pr:stab-init}
\min_{h\in\cH} \enscond{\risk(h)}{\cC_{\Pi}(h; t) \leq \delta} \enspace. \tag{$\cP$}   
\end{equation}
\textbf{Convex mixtures.} 
In our framework, fairness is characterized through the behaviour of the prediction function across a family of perturbed distributions. This makes it natural to combine predictors trained on different distributions, especially given that model averaging has already been shown to be effective for enforcing fairness in the classical setting~\citep{fermanian2025fair}.
% \begin{definition}[Convex mixture]
%  For functions $h_1,\cdots,h_K$ and for a vector $\balpha = (\alpha_1,\cdots,\alpha_K) \in \Delta_K$, we define a convex mixture as $h_\balpha = \sum_{k=1}^K \alpha_k h_k$.
% \end{definition}

For $\pi_1,\cdots, \pi_K \simiid \Pi$, let $\eta_1,\cdots,\eta_K$ be (unconstrained)\textit{ Bayes regression functions} minimizing the risks on respective distributions $\cD_{\pi_1},\cdots,\cD_{\pi_K}$, that is $\eta_k \in \argmin_{h}{\risk_{\cD_k}(h)}\enspace,\text{for } k \in [K] \enspace.$ For $\balpha\in\Delta_K$, we define a \textit{convex mixture} $\bdeta_\balpha \eqdef \sum_{k=1}^K \alpha_k \eta_k$. 

Rather than optimizing the true risk of the mixture directly, we focus on its linearized convex surrogate, defined as $\overline \risk (\bdeta_\balpha) \eqdef \sum_{k=1}^K \alpha_k \risk (\eta_k)$. Since the loss is convex in prediction, Jensen's inequality implies that $\risk(\bdeta_\balpha) = \risk (\sum_{k=1}^K \alpha_k \eta_k) \leq \sum_{k=1}^K \alpha_k \risk(\eta_k) = \overline \risk(\bdeta_\balpha)\enspace,$ showing that this surrogate upper-bounds the actual risk of the convex mixture. Given a $t\in[0,1]$, the problem~\eqref{pr:stab-init} translates as 
 \[
 \label{pr:stab}
 \balpha^\star \in \argmin_{\balpha\in\Delta_K} \enscond{\overline{\risk}(\bdeta_\balpha)}{\cC_{\Pi}(\bdeta_\balpha; t) \leq \delta} \enspace. \tag{$\cP_\balpha$}
 \]
\textbf{Empirical problem.} Let us define the empirical counterparts of the quantities of our interest and consider the empirical problem of~\eqref{pr:stab}. We define the empirical risk of $h$ on an \iid sample $\cD^N=\{\bZ_1,\cdots,\bZ_N\}$ as $\hat\risk_{\cD^N}(h)\eqdef\tfrac{1}{N}\sum_{i=1}^N\ell(h(\bX_i),Y_i)$.
Let $\hat\eta_k \in \argmin_{h}{\hat\risk_{\cD^N_{\pi_k}}(h)}$ estimate $\eta_k$ for all $k\in[K]$ on the perturbed samples of $\cD^N$. For $\balpha\in\Delta_K$, we define $\hat{\bdeta}_\balpha \eqdef \sum_{k=1}^K \alpha_k \hat\eta_k$.
% Let $\hat{\eta}_1,\cdots,\hat{\eta}_K$ be the estimates of $\eta_1,\cdots,\eta_K$ on respective perturbations of $\cD^N$, \ie $\hat\eta_k \in \argmin_{h}{\hat\risk_{\cD^N_{\pi_k}}(h)}\enspace,\forall k \in [K]$. For $\balpha\in\Delta_K$, we define $\hat{\bdeta}_\balpha \eqdef \sum_{k=1}^K \alpha_k \hat\eta_k$.
\begin{assumption}(Regression rates)
\label{ass:reg-rate}
It holds that $\max_{k \leq K} \Exp|{\hat{\eta}_k - \eta_k}| \leq \rho_{N}$.
\end{assumption}
\begin{remark}(On regression rates)
\label{rem:reg-rates}
Assumption~\ref{ass:reg-rate} can be satisfied with $\rho_N \leq \tilde\bigO(\nicefrac{1}{\sqrt{N}})$ for parametric cases, and even with faster rates under certain assumptions.
For example, under the assumption, that there exist $\beta \in \bbN \setminus \{0\}$ and $L > 0$, such that $ |\eta(\bx) - \eta_{\bx}(\bx')| \leq L \|\bx - \bx'\|^\beta\,,$ where $\eta_{\bx}$ is Taylor polynomial of degree $\beta$ at $\bx$, and under an additional assumption that the support of $\Prob_{\bX}$ is $\cX = [-1, 1]^d$ and it admits density uniformly lower and upper bounded on $\cX$, then for some constants $C_1,C_2$, it holds that $\sup_{\Prob}\Prob_n(|\eta(\bX) - \hat\eta(\bX)| \geq \epsilon) \leq C_1 \exp(C_2 N^{-\frac{2\beta}{2\beta+d}}\epsilon^2)$ \citep[Theorem~3.2]{tsybakov2007fast}. Here, Assumption~\ref{ass:reg-rate} is satisfied with $\rho_N \leq \tilde\bigO(N^{-\frac{\beta}{2\beta+d}})$.
\end{remark}
\begin{lemma}[Mixture rates] Under Assumption~\ref{ass:reg-rate} and Claim~\ref{claim:bound-w}, $\Exp|{\hat{\bdeta}_\balpha - \bdeta_\balpha}| \leq \rho_{N}$ for all $\balpha\in\Delta_K$.    
\end{lemma}

Let $\cD^n = \{\bZ_1,\cdots,\bZ_n\}$ be \iid samples from $\cD$, and $\Pi^{m} = \{\pi'_1,\cdots,\pi'_m\}$ be \iid samples from $\Pi$. For a fixed $\pi'$, we define $\hat{\Phi}_n(h;\pi') \eqdef \tfrac{1}{n}\sum_{i=1}^n {w}_\pi(\bZ_i)\phi_{h}(\bZ_i)$. Let $\hat\bA\in\bbR^{m\times K}$ be a matrix with elements $\hat A_{j,k}=\hat\Phi_n(\hat\eta_k,\pi_j)$. Notice, that we can rewrite the empirical constraint as 
$\hat{\cC}_{n,m}(\hat{\bdeta}_\balpha,t)=\tfrac{1}{m}\|{\hat\bA\alpha-t \mathbf{1}}\|_2^2 \enspace.$
Denoting $\hat{r}_k \eqdef \hat{\risk}(\hat\eta_k)$ and $\hat\br=(\hat{r}_1,\cdots,\hat{r}_K)$, we obtain the following second-order cone program (SOCP)
\begin{equation}
\label{pr:socp}
\tag{$\cP_\texttt{SOCP}$}
\hat{\balpha} \in \argmin_{\balpha\in\Delta_K} \enscond{\hat \br^\top \balpha}{\big\|{\hat\bA\balpha-t \mathbf{1}}\big\|_2 \leq \sqrt{m\delta}} \enspace.
\end{equation}
\textbf{Feasibility and existence.} The problem~\eqref{pr:socp} has a continuous linear objective, and the simplex $\Delta_K$ is compact and convex. Therefore, once feasibility is established, existence of an optimizer follows from the Weierstrass extreme value theorem. In classical fairness problems, (strict) feasibility is ensured by constant functions, but in our mixture setting one cannot generally claim that a constant function belongs to the convex hull of $\{\eta_k\}_{k=1}^K$ (or $\{\hat\eta_k\}_{k=1}^K$). However, if we augment both population and empirical dictionaries by constant functions $h_{K+1}\equiv1$ and $h_{K+2}\equiv0$, and choose $\balpha=(0,\cdots,0,t,1-t)$, yielding a constant predictor $h_\balpha\equiv t$, we will have $\cC_\Pi(t,t)=0<\delta$ (or $\hat\cC_{n,m}(t,t)=0<\delta$). Note, that these two additional predictors do not affect our general theoretical results, thus we omit them from our discussion, and impose strict feasibility as an assumption.
\begin{assumption}[Strict feasibility] \label{ass:feas} There exists $\balpha^0 \in \Delta_K$ \st $\cC_\Pi(\bdeta_{\balpha^0};t) < \delta$.
\end{assumption}
\subsection{Constructing perturbations} 
%The construction of perturbations is twofold. First, we choose a distribution $\Pi$ together with a sampling rule. Then, for each sampled $\pi'\sim\Pi$, we construct the perturbed distribution $\cD_{\pi'}$.
Let us denote by $p_g \eqdef \Prob (G = g)$ and $p_{g,y} \eqdef \Prob (G = g \mid Y=y)$ for all $g,y \in \{0,1\}\enspace.$
\begin{assumption}[Bounded group probabilities]
\label{ass:bound-p} Let $p_g, p_{g,y}\in[\ubar{p}, 1 - \ubar{p}]$, with $\ubar{p} \in (0,\nicefrac{1}{2})$.
\end{assumption}
\begin{assumption}[Bounded perturbations]
\label{ass:bound-pi} Let $\Pi$ be supported on $[\ubar{\pi},1-\ubar{\pi}]$, with $\ubar{\pi}\in(0,\nicefrac{1}{2})$.
\end{assumption}
The construction of perturbations is twofold. First, we consider \textit{stratified sampling} from $\Pi$. Letting $\Pi=U(\ubar{\pi},1-\ubar{\pi})$, we partition $[\ubar{\pi},1-\ubar{\pi}]$ uniformly into intervals $I_1,\cdots,I_m$ of length $\tfrac{(1-2\ubar{\pi})}{m}$, and sample $\pi'_j \sim U(I_j)$ independently for $j=1,\cdots,m$. The stratification yields faster convergence rates, as discussed in the next section.
% We are able to show a particularly attractive property, that is for a Lipschitz function $f:[\ubar{\pi},1-\ubar{\pi}]\to\bbR$, it holds \textit{with high probability} that $
% \Big|\frac{1}{m}\sum_{j=1}^m f(\pi_j') - \Exp_\Pi[f(\pi)]\Big| \leq \tilde\bigO(\frac{1}{m^{\nicefrac{3}{2}}}) $. This property allows us to get faster convergence rates from the error that occurs from the sampling of the perturbation parameters. This is discussed in Appendix~\ref{app:rates:pert}.
% \textbf{Constructing $\cD_\pi$ through importance weights.} 
Then, for a perturbation parameter $\pi\in\Pi$, the corresponding distribution shift is represented through \textit{importance-weights}, defined by the Radon-Nikodym derivative as
$
w_\pi(\bz) \eqdef {\tfrac{d \cD_\pi}{d \cD}}(\bz) > 0 \enspace.
$
For our examples, Assumptions~\ref{ass:bound-p},\ref{ass:bound-pi} ensure that $\cD_\pi\ll\cD$, and we can explicitly compute $w_\pi$ (Appendix~\ref{app:weights}). Applying Radon-Nikodym theorem gets us $
\Phi(h;\pi) = \Exp_{\cD_\pi}[\phi_h(\bZ)] = \Exp[w_\pi(\bZ)\phi_h(\bZ)] \enspace. 
$
 % From  In our examples, we can explicitly compute $w_\pi$ (Appendix~\ref{app:weights}). 
%  For example, under Perturbation~\ref{pert:G}, considering the factorization $\Prob(\bX, G, Y) = \Prob(G)\Prob(\bX,Y \mid G)$, we get $
% w_\pi (g) = \frac{\pi}{p_1}\ind{g=1} +
% %   \frac{1-\pi}{p_0}\ind{g = 0} \enspace.$ Or, under Perturbation~\ref{pert:G-1}, considering the factorization $\Prob(\bX, G, Y) = \Prob(Y)\Prob(G \mid Y)\Prob(\bX \mid Y, G)$, we get $w_\pi (g, y) = \frac{\pi}{p_{1,1}} \ind{(g,y) = (1,1)} + \frac{1-\pi}{p_{0,1}} \ind{(g,y) = (0,1)} + \ind{y=0}\enspace.$ 
% Moreover, for any measurable function $f(\cdot)$, when $\cD_\pi\ll\cD$ (which is ensured by Assumptions~\ref{ass:bound-p}-\ref{ass:bound-pi} for our examples), it holds that $\Exp_{\cD_\pi}[f(\bZ)] = \Exp[w_\pi(\bZ)f(\bZ)]$. In particular, considering Definition~\ref{def:Phi}, we can state that
% $
% \Phi(h;\pi) = \Exp_{\cD_\pi}[\phi_h(\bZ)] = \Exp[w_\pi(\bZ)\phi_h(\bZ)] \enspace. 
% $
This means, that instead of resampling, we can just compute and plug in the respective importance weights onto our expressions. 
\begin{claim}[Bounded weights]
\label{claim:bound-w} Under Assumptions~\ref{ass:bound-p},\ref{ass:bound-pi}, $\exists W>0$ \st $w_\pi (\bZ) < W, \forall 
(\pi,\bZ)\in\Pi\times\Z$.
%\pi\in\Pi\enspace,\enspace \forall \bZ \in\Z$.
\end{claim}
\begin{algorithm}[t]
\caption{Pseudo-Algorithm}
\label{alg:main-pseudo}
\begin{algorithmic}[1]
\REQUIRE Samples $\{\bZ_i\}_{i=1}^n\simiid \cD$, perturbation parameter law $\Pi$, target $t$, threshold $\delta$.
\STATE Sample $\pi_1,\cdots,\pi_K \simiid \Pi$ 
\FOR{$k=1$ \textbf{to} $K$}
  \STATE Train $\hat\eta_k$ by weighted ERM on $\{\bZ_i\}$ with weights $w_{\pi_k}(\bZ_i)$.
\ENDFOR
\STATE Sample $\pi'_1,\ldots,\pi'_m\simiid \Pi$.
\STATE Build $\hat{\bA}$, with $\hat{A}_{j,k}=\hat\Phi_n(\hat\eta_k,\pi'_j)=\frac{1}{n}\sum_{i=1}^n w_{\pi'_j}(Z_i)\phi_{\hat\eta_k}(Z_i)$; and $r_k = \hat\risk_n(\hat\eta_k)$
\STATE Solve the SOCP \eqref{pr:socp} to obtain $\hat\balpha$.
\STATE Set $\hat\bdeta_{\hat\balpha}=\sum_{k=1}^K \hat\alpha_k \hat\eta_k$.
\end{algorithmic}
\end{algorithm}
\subsection{Convergence rates}
We now turn to the convergence analysis of the finite procedure, aiming to understand how the empirical solution $\hat{\bdeta}_{\hat\balpha}$ is approaching its population counterpart $\bdeta_{\balpha^\star}$. There are three sources of error in our framework, each contributing in its own rate. Under Assumption~\ref{ass:reg-rate} and Remark~\ref{rem:reg-rates}, \textit{estimating regression functions} $\eta_k$ from $N$ samples of $\cD^N$ yields to an error $\rho_N \asymp  N^{-\frac{\beta}{2\beta+d}}$, that under convexity transfers to the mixture $\bdeta_\balpha$ as well. Linearity of $\phi_h$ in $h$, combined with Claims~\ref{claim:phi-bd},\ref{claim:bound-w} leads to overall rate of $| \Phi(\hat{\bdeta}_\balpha; \pi) -\Phi(\bdeta_\balpha; \pi) | \leq \delta_N \asymp \rho_N$. For a fixed $\pi'$, under Claim~\ref{claim:bound-w}, and with Bernstein's inequality, the \textit{empirical approximation} of $\Phi(\hat\eta_k; \pi')$ from $n$ samples from $\cD^n$ contributes $\delta_n \asymp n^{-\frac{1}{2}}$. The third error comes from \textit{perturbation parameter sampling}, where for any Lipschitz function of $\pi$, stratification of $[\ubar{\pi}, 1- \ubar{\pi}]$ into $m$ intervals replaces the usual $m^{-\frac{1}{2}}$ by $\delta_m \asymp m^{-\frac{3}{2}}$. The \textit{risk analysis} relies on a value continuity argument, that follows from Assumption~\ref{ass:feas}. Under Assumption~\ref{ass:loss:bd}, the approximation of the linearized risk introduces an additional error $\delta_{n(R)}\asymp n^{-\frac{1}{2}}$. 
The combination of the introduced errors leads to the following (informal) theorem. The formal versions of all the statements from this section, along with their proofs, are given in Appendix~\ref{app:conv-rates}.
\begin{theorem}[Informal]\label{thm:main-rates} Suppose Assumptions~\ref{ass:loss:bd}--\ref{ass:bound-pi} hold. Then, conditionally on $\cD^N$ and with ``high probability'' over $(\cD^n,\Pi^m)$, the following hold.
%\begin{enumerate}
    %\item 
    
    \quad\text{1. Constraint.} Let $\hat{\balpha} \in \Delta_K$ \st $\hat{\cC}_{n,m}(\hat{\balpha}; t) \leq \delta$. Then, $\cC_\Pi(\hat{\bdeta}_{\hat{\balpha}}; t)
    \leq (\sqrt{\delta} + \delta_n)^2 + \delta_m \enspace.$
    %\item 
    
    \quad\text{2. Unfairness.} It holds that $\unf(\hat{\bdeta}_{\hat{\balpha}})
    \leq \nicefrac{(\sqrt{\delta} + \delta_n + \sqrt{\delta_m})}{\sqrt{\Var_\Pi(\pi)}} \enspace.$
    %\item 
    
    \quad\text{3. Uniform Deployment.} For {all} $\pi' \in [\ubar{\pi},1-\ubar{\pi}]$, it simultaneously holds that 
    
    \qquad $|\Phi(\hat{\bdeta}_{\hat\balpha};\pi') - t| \leq (\sqrt{\delta} + \delta_n  + \sqrt{\delta_m})
    \cdot \left(1 + \nicefrac{|\pi' - \Exp_\Pi[\pi]|}{\sqrt{\Var_\Pi(\pi)}}\right) \enspace.$
    %\item 
    
    \quad\text{4. Risk.} There exists a constant $0<C<\infty$, \st it holds that

    \qquad${\overline{\risk}} (\hat\bdeta_{\hat\balpha}) - {\overline{\risk}} (\bdeta_{\balpha^\star}) \leq L\cdot\delta_N  + C \cdot (2\sqrt{\delta}(\delta_n + \delta_N) + (\delta_n + \delta_N)^2 + \delta_m) + 2\delta_{n(R)}\enspace.$
%\end{enumerate}
\end{theorem}
Theorem~\ref{thm:main-rates} shows that the convergence rates for the stability constraint and risk in are within $\tilde\bigO({n^{-\frac{1}{2}}} + {m^{-\frac{3}{2}}} + N^{-\frac{\beta}{2\beta+d}})$. It also shows that our reformulated problem~\eqref{pr:stab} doesn't yield any convergence loss for unfairness, whose rate is $\tilde\bigO({n^{-\frac{1}{2}}} + {m^{-\frac{3}{4}}} + N^{-\frac{\beta}{2\beta+d}})$, and provides additional deployment guarantee. The deployment gap is the largest at the extreme points $\pi'\in\{0,1\}$, corresponding to single-group populations. Unlike a classifier trained under a classical fairness constraint (\eg DP,EO) at the training marginal $p_1$, that only controls the unfairness, and does not prevent $\Phi(h;\pi')$ from drifting linearly as $\pi'$ changes, the solution $\hat\bdeta_{\hat\balpha}$ of~\eqref{pr:socp} controls both the \emph{slope} $\gamma(\hat{\bdeta}_{\hat{\balpha}})$ and the
\emph{absolute level} around $t$, yielding a strictly stronger and distribution-free deployment guarantee. This also highlights the role of $\Pi$ as a design parameter. Since the guarantee improves with $\Var_\Pi (\pi)$, choosing broader perturbation law leads to stronger fairness guarantees. Finally, the number of classifiers $K$ appears only in logarithmic factors in the $\tilde\bigO$ bounds and therefore does not affect the convergence rates.
\section{Numerical experiments}\label{sec:exp}

We evaluate our approach (referred to as STABLE) on two benchmark datasets, namely Adult Income~\citep{adult1996} and COMPAS~\citep{angwin2016machine} under both DP and EO constraints, where we consider binary classification with gender and race as sensitive attributes. We also consider a large-scale deployment scenario using ACS Income~\citep{ding2021retiring}, a survey of the US adult population covering all $50$ states, each with a distinct racial composition. We train on California and evaluate on 11 held-out states.\footnote{The code is available at~\url{https://github.com/laclauc/stable-fairness}.}. 

\paragraph{Baselines} 
STABLE mixes $K=30$ weighted-ERM logistic regressions, the two constant predictors and, except on COMPAS under DP, $2K_c=20$ logistic regressions trained on label-flipped data; this choice is made on validation data (Appendix~\ref{app:eval}).  We compare against ERM under constraints \cite{pmlr-v80-agarwal18a}, PostProc \cite{chzhen2025randomizedmulticlassclassificationconstraints}, Shifty \cite{giguere2022fairness}. We also report STABLE$_{\text{DRO}}$, an
ablation that replaces our objective with an SA-DRO-style one
\cite{lei2024on} on the same dictionary and perturbations as STABLE
(Appendix~\ref{app:experiments}). Demographic shift is simulated via importance weighting (see Appendix~\ref{app:eval}).  
The target $t$ is set as the mean prediction of ERM on the training set (DP) or the average group-conditional TPR of ERM (EO). Rather than a single operating point per method, we sweep each method over a common grid of fairness budgets and report the resulting accuracy--unfairness and accuracy--deployment-gap frontiers, averaged over 10 random seeds (5 on ACS Income; protocol in Appendix~\ref{app:eval}). For STABLE, a deeper analysis of hyperparameters and the impact of $t$ values on training is conducted in Appendices~\ref{app:hyperparameter} and~\ref{app:t-study}.

\begin{figure}
    \centering
    \includegraphics[width=0.98\linewidth]{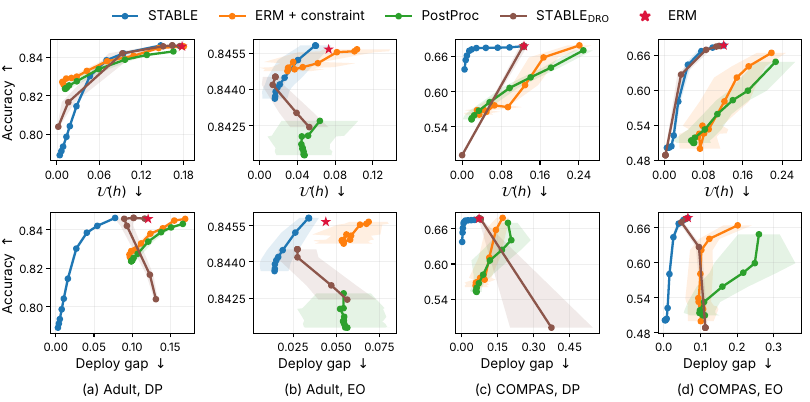}
    \caption{Accuracy--unfairness (top) and accuracy--deployment-gap (bottom)
  frontiers on Adult and COMPAS. Each point corresponds to one fairness budget
  of the common grid, averaged over 10 seeds; shaded regions show $\pm$ one
  standard deviation along the horizontal axis.}
    \label{fig:all-pareto}
\end{figure}

\begin{figure}
\centering    \includegraphics[width=\linewidth]{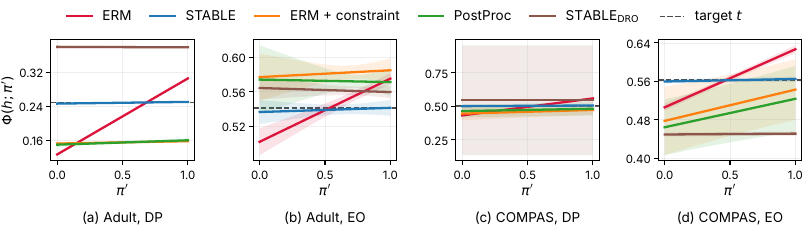}
    \caption{Prediction rate $\Phi(h;\pi')$ as a function of the deployment proportion $\pi'$ at the tightest fairness budget. Lines show the mean over 10 seeds, shaded regions $\pm$ one standard deviation; the dashed line is the target $t$}
    \label{fig:shift-results}
\end{figure}

\textbf{Results.} Figure~\ref{fig:all-pareto} reports the Pareto frontiers of all methods (additional plots including Shifty~\cite{giguere2022fairness} are presented in Appendix~\ref{app:sec:shifty}). Overall, STABLE is highly competitive with the other models. STABLE outperforms all the models in terms of accuracy--deployment gap trade-off, and achieves the lowest deployment gap among all the baselines across all settings with a large margin. The accuracy--fairness trade-off, on the other hand, varies across settings. On COMPAS, STABLE outperforms ERM+constraint and PostProc both on DP and EO. On the former, for STABLE$_{\text{DRO}}$, the least unfair point is a constant predictor whose level varies across seeds while on EO it gets a comparable accuracy-fairness trade-off but a large deployment gap. On Adult under EO, all methods lie within 0.005 accuracy of each other and their differences in $\mathcal{U}(h)$ are within seed variability. Finally, on Adult under DP, the trade-off is less favourable for STABLE, where at a lower unfairness it has a larger accuracy loss compared to other models. Figure~\ref{fig:shift-results} illustrates the prediction rate \wrt $\pi'$. The prediction rate for classical unconstrained ERM dramatically drifts across population compositions. For PostProc and ERM+constraint the angle of the slope depends on the setting, however, even when the slope is small, the rate still remains far from $t$. The ablation with SA-DRO highlights the role of the level $t$ and the SOCP formulation: while the prediction rate for STABLE$_{\text{DRO}}$ is nearly flat in $\pi'$, it still remains away from $t$. 

Overall, everything is consistent with the theory: the SOCP constraint controls both the slope $\gamma(h)$ and the absolute level of $\Phi(h;\pi')$ around $t$, which is strictly stronger than the classical fairness constraint and shrinks the feasible set accordingly. Moreover, SOCP optimizes a linearized surrogate of the risk, which, despite upper-bounding the true mixture risk, introduces a linearization gap that may occasionally translate into reduced accuracy. 

\iffalse The accuracy-fairness trade-off varies across settings. Overall, STABLE is highly competitive with the other models (Table~\ref{tab:main_results}). On Adult, it outperforms all baselines under EO, while under DP it suffers a notable loss in accuracy, despite achieving the lowest unfairness and deployment gap. On COMPAS, it outperforms all baselines under DP and, under EO, still achieves the best accuracy and deployment gap, although with higher unfairness. Overall, this is consistent with the theory: the SOCP constraint controls both the slope $\gamma(h)$ and the absolute level of $\Phi(h;\pi')$ around $t$ (Figure~\ref{fig:shift-results}), which is strictly stronger than the classical fairness constraint and shrinks the feasible set accordingly. Moreover, SOCP optimizes a linearized surrogate of the risk, which, despite upper-bounding the true mixture risk, introduces a linearization gap that may occasionally translate into reduced accuracy. 
\fi 
%Shifty and POST-PROC (to confirm) on COMPAS returns near-trivial models (mean predicted probability close to 0 or 1) and satisfy the DP constraint by predicting almost exclusively the negative (or positive) class. This behavior, documented by \citet{giguere2022fairness} on small datasets, illustrates a fundamental limitation of certification-based approaches under data scarcity.

For ACS, Figure~\ref{fig:acs} compares the methods at a fixed trade-off: the same budget $\delta_0 = 0.01$ for all methods, and STABLE$_{\text{DRO}}$ at its least unfair value of $\lambda$. Additional results for both constraints are reported in Appendix~\ref{app:sec:acs}. We remark that STABLE remains the closest to $t$ across all held-out states (Figure~\ref{fig:acs}). PostProc and {ERM+constraint} introduce a systematic geographic bias: their slope $\pi_{\text{state}} \mapsto \Phi(h;\pi_{\text{state}})$ is negative, causing them to drift below $t$ in high-proportion states. The slope of STABLE$_\text{DRO}$ is similar the one of STABLE, however still remains far from $t$. This illustrates that controlling static unfairness does not preclude deployment bias. The not completely flat prediction rate for STABLE is consistent with the fact that the population composition shift doesn't exactly follow Perturbation~\ref{pert:G} from theory. On the accuracy--unfairness trade-off, STABLE overall attains accuracies higher or comparable to other models. Results for EO are deferred to the Appendix~\ref{app:sec:acs}.

\begin{figure}
\centering
\includegraphics[width=0.9\linewidth]{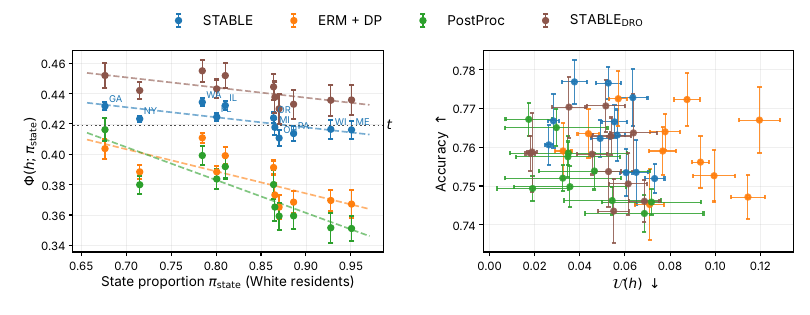}  
\caption{Geographic generalization on ACS Income, at a fixed fairness budget.
(left) Prediction rate $\Phi(h;\pi_{\rm state})$ against the proportion of White residents in each held-out state; dashed lines: linear fits, dotted line: target $t$. (right) Accuracy against $\mathcal{U}(h)$. Each point is one state, mean $\pm$ std over 5 seeds.}\label{fig:acs}
\end{figure}

\iffalse
\section{Illustration with Linear Models}

To make the notion of fairness as stability more concrete, we consider
a simple linear predictor. Let $x \in \mathbb{R}^d$ be the input
features and $y \in \mathbb{R}$ the target. We consider linear
regression:

\[
h_\theta(x) = \theta^\top x,
\]

where $\theta \in \mathbb{R}^d$ is learned by minimizing the empirical
risk on a dataset $\mathcal{D} = \{(x_i, y_i)\}_{i=1}^n$:

\[
\hat{\theta} = \arg\min_\theta \frac{1}{n} \sum_{i=1}^n \ell(h_\theta(x_i), y_i),
\]

with $\ell$ a loss function (e.g., squared error).

\subsection{Stability Measure}

We define a stability measure $\Delta$ between predictors:

\[
\Delta(h_\theta, h_{\theta_\pi}) = \frac{1}{n} \sum_{i=1}^n |h_\theta(x_i) - h_{\theta_\pi}(x_i)|^2.
\]

This captures the average change in prediction on the original dataset when the model is retrained on a perturbed dataset.

\subsection{Fairness Analysis of Linear Models}
Let us consider the simple case where $h$ is a linear predictor. 
\fi
\section{Conclusion and perspectives}\label{sec:concl}
We introduced a stability-based perspective on fairness and showed that both individual and group fairness arise from it through appropriate perturbation classes. We then introduced a stronger constraint that jointly controls group fairness and the stability of the prediction rate across perturbations, and showed that when considering a linearized risk, the problem admits a tractable SOCP solution. We further established convergence guarantees from empirical to population quantities and illustrated the practical relevance of the problem through numerical experiments and comparisons to various benchmarks. One limitation is that our method addresses a specific tractable surrogate rather than the original problem in full generality. A natural next step is therefore to derive a more direct solution of the original problem and to establish end-to-end guarantees.
% acknowledgments
\begin{ack}
This work is supported by Hi! PARIS and ANR/France 2030 program (ANR-23-IACL-0005) and by the PEPR IA FOUNDRY (ANR-23-PEIA-0003). This work was also supported by the Caisse des Dépôts through the Operational AI Ethics Chair.
\end{ack}
% bibliography
%\newpage
\bibliography{bibliography}
\bibliographystyle{plainnat}
%\section*{References}
% ------------ appendix ----------
\newpage 
\appendix
\section*{Organization of the Supplementary Material}
The Supplementary Material is organized as follows.
\begin{itemize}
    \item Appendix~\ref{app:disc} provides additional literature review and further discussion.
    \item Appendix~\ref{app:proofs} contains the main proofs establishing the equivalence between existing fairness notions and our stability-based definition, together with additional discussion.
    \item Appendix~\ref{app:rates} provides the explicit forms of the importance weights, the proofs of the convergence results for the constraints and the risk, along with some auxiliary lemmas needed in the analysis.
    \item Appendix~\ref{app:eodds} extends the framework to Equalized Odds. We introduce the relevant notions and prove the associated equivalence result with stability.
    \item Appendix~\ref{app:experiments} is dedicated to numerical experiments. We provide additional implementation details and conduct further experiments.
\end{itemize}

\section{Discussions on related works}
\label{app:disc}
Machine learning models are increasingly used in high-stakes decisions, such as hiring, healthcare, and criminal justice. In such settings, biases already present in the data may be inherited or even amplified by the resulting models, with potentially serious consequences for the individuals affected. Algorithmic fairness seeks to understand, quantify, and mitigate these effects. Since this question also involves philosophical and social considerations, it naturally lies at the intersection of several fields; in this paper, however, we focus on its mathematical aspect and do not aim to resolve the broader question of what fairness ought to be.
\subsection{Stability and learning theory}
\label{app:stability}
Many notions of stability have been introduced in the Statistics and Machine-Learning literature to define and quantify the \textit{stability property} of a statistical learning algorithm, meaning roughly that a slight change in the data (distribution) has no significant effect on the decision function produced by the algorithm. Depending on the definition of stability chosen, it is possible to use concentration results to deduce error bounds for decision rules from the stability property satisfied by the algorithm producing them, \textit{e.g.}, ERM with a certain cost function and/or with additive penalization, bagging. This idea is not new at all and dates back at least to the late 70's in Statistics, refer to \cite{10.1214/aos/1176344196}, \cite{10.1109/TIT.1979.1056032}, \cite{1056087} and \cite{10.1162/089976699300016304} for instance. The concept of stability has more recently been revisited in new forms by the machine learning community. See in particular the concept of \textit{uniform stability} in \cite{10.1162/153244302760200704} yielding exponential bounds and the refinements in \cite{10.5555/2073876.2073909}. More generally, this line of research has been adopted by many authors, too numerous to be listed in an exhaustive fashion, dealing with various approaches (\textit{e.g.} bagging in \cite{soloff2024baggingprovidesassumptionfreestability}) and issues (clustering, latent variable analysis, ranking).

\subsection{Comparison with Distributionally Robust Optimization}\label{app:dro}
While our framework bears a resemblance to distributionally robust optimization (DRO), the two approaches differ in both structure and the guarantees they provide, with interesting connections that we now discuss.

\paragraph{DRO reminder.} Given a divergence $d$ (here KL) and radius
$\varepsilon > 0$, define the uncertainty set $\mathcal{U}_\varepsilon = \{\mathcal{D}': \mathrm{KL}(\mathcal{D}' \|
\mathcal{D}) \le \varepsilon\}$. 
DRO solves
\[
  \min_h \sup_{\mathcal{D}'\in\mathcal{U}_\varepsilon}
  \mathbb{E}_{\mathcal{D}'}[\ell(h(\bX),Y)].
\]

Our demographic shifts $\mathcal{D}_\pi$ (Perturbation~\ref{pert:G}) satisfy $\mathcal{D}_\pi \in \mathcal{U}_\varepsilon$ whenever $\varepsilon \ge \mathrm{KL}(\mathcal{D}_\pi \| \mathcal{D})$. Since $\mathcal{D}_\pi$ only resamples $G$ while keeping $P(\bX, Y \mid G)$ fixed, the KL reduces to the KL between two Bernoulli distributions:
\[
  \mathrm{KL}(\mathcal{D}_\pi \| \mathcal{D})
  = \pi \log\frac{\pi}{p_1} + (1-\pi)\log\frac{1-\pi}{1-p_1}.
\]
To cover all deployment proportions $\pi' \in [\bar{\pi},1-\bar{\pi}]$, DRO requires a radius at least
\[
  \varepsilon^*(\bar{\pi})
  = \max_{\pi \in [\bar{\pi}, 1-\bar{\pi}]}
  \left[ \pi \log\frac{\pi}{p_1} + (1-\pi)\log\frac{1-\pi}{1-p_1} \right],
\]
which is attained at one of the endpoints $\pi \in \{\bar{\pi},
1-\bar{\pi}\}$ by convexity of the KL divergence. Since $\mathrm{KL}(\mathrm{Bernoulli}(\pi) \| \mathrm{Bernoulli}(p_1))$ is continuous and strictly positive for $\pi \neq p_1$, we have $\varepsilon^*(\bar{\pi}) > 0$ uniformly in $p_1 \in [\bar{p}, 1-\bar{p}]$; in particular, $\varepsilon^* = \Omega(1)$ whenever $|\bar{\pi} - p_1| \ge c$ for some constant $c > 0$. Table~\ref{tab:kl_radius} illustrates
representative values.

\begin{table}[h]
\centering
\caption{Minimum DRO radius $\varepsilon^*(\bar{\pi})$ required to
cover all demographic shifts $\pi' \in [\bar{\pi}, 1-\bar{\pi}]$,
for varying training group proportion $p_1$ and perturbation lower
bound $\bar{\pi}$.}
\label{tab:kl_radius}
\begin{tabular}{lccc}
\toprule
$p_1 \;\backslash\; \bar{\pi}$ & $0.1$ & $0.2$ & $0.3$ \\
\midrule
$0.2$ & $1.15$ & $0.83$ & $0.58$ \\
$0.3$ & $0.79$ & $0.53$ & $0.34$ \\
$0.5$ & $0.37$ & $0.19$ & $0.08$ \\
\bottomrule
\end{tabular}
\end{table}

At radius $\varepsilon = \varepsilon^* = \Omega(1)$, the uncertainty set
$\mathcal{U}_\varepsilon$ is sufficiently large to contain distributions that arbitrarily corrupt the feature mechanism $P(\bX \mid G)$ and the label mechanism $P(Y \mid G)$. Consequently, the worst-case risk over $\mathcal{U}_\varepsilon$ is controlled by corruptions of $P(\bX|G)$ and $P(Y|G)$ that are unrelated to demographic fairness, hence making the bound uninformative about the performance under such demographic shift. 

%approaches the trivial upper bound of $1$, and the DRO guarantee reduces to
%\[
%  \sup_{\mathcal{D}' \in \mathcal{U}_{\varepsilon^*}}
%  \mathbb{E}_{\mathcal{D}'}[\ell(h(X), Y)] \approx 1,
%\]
%which holds for \emph{any} predictor and conveys no information about
%the actual performance under demographic shift.

\paragraph{Why our framework avoids this problem.}
The uncertainty set $\mathcal{U}_\varepsilon$ contains distributions that alter $P(\bX \mid G)$, $P(Y \mid G)$, and other aspects of the data generating process entirely irrelevant to demographic fairness. Our framework instead restricts to the \emph{one-dimensional
manifold} $\{\mathcal{D}_\pi : \pi \in [\bar{\pi}, 1-\bar{\pi}]\}$, which captures exactly the shifts of interest while preserving $P(\bX, Y \mid G)$. This structural specificity combined with the affine form $\Phi(h;\pi) = \alpha(h) + \pi\gamma(h)$, enables simultaneous control of slope and absolute level via $\mathcal{C}_\Pi(h;t)$. As a result, Theorem~\ref{thm:main-rates} provides guarantees that converge to zero as $n, m \to \infty$ -- something no DRO bound with radius $\varepsilon = O(1)$ can achieve.
\paragraph{Connection to Group DRO and SA-DRO.}
Our framework is more closely related to structured DRO approaches such as Group DRO \citep{Sagawa*2020Distributionally}.  In the binary case, Group DRO minimizes $\max_{\lambda \in [0,1]} \lambda \mathbb{E}_{D_1}[\ell(\cdot)] + (1-\lambda)\mathbb{E}_{D_0}[\ell(\cdot)]$, which is exactly the worst-case risk over our Perturbation~\ref{pert:G} family $\{\mathcal{D}_\pi : \pi \in [0,1]\}$. The two frameworks thus share the same perturbation structure for demographic parity. The differences are twofold. First, the objectives differ: Group DRO minimizes only the worst-case risk, while \textsc{Stable} controls the variance of $\Phi(h;\pi)$ around a target level $t$, which we show to be formally equivalent to fairness constraints. Second, Perturbation~\ref{pert:G-1}, which resamples $G$ within the positive class to capture equal opportunity, has no natural Group DRO counterpart, as Group DRO cannot vary $P(G=1\mid Y=1)$ independently of $P(G=1\mid Y=0)$.

A closely related work is SA-DRO \citep{lei2024on}, which applies DRO directly on the marginal distribution of the sensitive attribute and adds an explicit fairness regularization term. The same structural differences with Group DRO apply here. Additionally, SA-DRO relies on soft regularization rather than a hard constraint, and therefore does not provide quantitative guarantees on the unfairness gap $\unf(h)$ or on the deployment gap.

\section{Proofs for Section~\ref{sec:equiv}}
\label{app:proofs}
\subsection{Individual fairness}
\label{app:proofs:if}
Recall that the Wasserstein $1$-distance between the pushforward of $h(\bX)$'s distribution under $\mathcal{D}$ and that under $\bP(\mathcal{D})$  based on the cost function $d_{\Y}$ is \[\inf_{\bX\sim \mathcal{D},\; \bX'\sim \bP(\mathcal{D})}\mathbb{E}[d_{\Y}(h(\bX),h(\bX')],\] where the infimum is taken over all couplings $(\bX,\bX')$ of the pair of distributions $\mathcal{D}$ and $\bP(\mathcal{D})$.

Proof of Proposition~\ref{prop:indiv-worst-stab}.
\begin{proposition*}[Proposition~\ref{prop:indiv-worst-stab}, IF and Stability]
Let $h:\X\to\Y$ be a prediction function and $d_{\Y}:\Y^2\to [0,\infty)$ a metric on predictions. The following assertions are equivalent.
\begin{enumerate}
    \item The predictive function $h$ satisfies the IF property with $(d_{\Y},d_{\X})$-Lipschitz constant $L<\infty$.
    \item For all $\rho >0$, the predictive function $h$ is $L\rho$-fair w.r.t. to the class $\cP_{\rho}$ of perturbation operators $\bP_{T}$, $T\in \mathcal{T}_{\rho}$, defined on the domain $\mathcal{M}=\{\delta_\bx:\; \bx\in \X\}$ of all point masses and taking $\Delta(h;\; (\mathcal{D},\bP))$ as the Wasserstein $1$-distance between the pushforward of $h(\bX)$'s distribution under $\mathcal{D}$ and that under $\bP(\mathcal{D})$  based on the cost function $d_{\Y}$:
\[
\sup_{(\bx,T)\in \X\times \mathcal{T}_{\rho}}d_{\Y}(h(\bx),h(T(\bx)))\leq L\rho \enspace.
\]
\end{enumerate}
\end{proposition*}
\begin{proof}
Let us show that $(1\Rightarrow2)$. It is easy to notice that
\begin{align*}
    \sup\limits_{(\bx,T)\in \X\times \mathcal{T}_{\rho}}d_\Y(h(\bx),h(T(\bx))) \leq \sup\limits_{(\bx,T)\in \X\times \mathcal{T}_{\rho}} L d_\X(\bx, T(\bx)) \leq L \rho \enspace.
\end{align*}

Now let us show that $(2\Rightarrow1)$. First, let us fix any $\bx,\bx'$, and denote $\rho':=d_\X(\bx,\bx')$. Let us also introduce the following transformation
\[
T_{\bx\to \bx'}(\bz) =
\begin{cases}
  \bx'\enspace, & if \quad \bz=\bx\\
  \bz\enspace, & otherwise \enspace. 
\end{cases}
\]
Choosing $\rho=\rho'$, we consider the set $\cT_{\rho'}$. We notice that $T_{\bx\to \bx'} \in \cT_{\rho'}$ and observe that
\begin{align*}
d_\Y(h(\bx), h(\bx')) 
&= d_\Y(h(\bx), h(T_{\bx\to \bx'}(\bx))) \\
&\leq \sup_{\bx\in\X} d_\Y (h(\bx),h(T_{\bx\to \bx'}(\bx))) \\
&\leq \sup_{T\in \cT_{\rho'}} \sup_{\bx\in\X}d_\Y(h(\bx),h(T(\bx))) \leq L \rho' =  L d_\X (\bx,\bx') \enspace.
\end{align*}
The proof is concluded.
\end{proof}
% \begin{remark}[On distributional form]
% We can translate the formulation in Proposition~\ref{prop:indiv-worst-stab} into its distributional form by considering the new joint distribution produced by a transformation $T \in \cT_\rho$. Recall, that the transformation acts only on $\bX$, leaving $Y$ and $G$ unchanged, thus, resulting in the pushforward $({T_\cD})_\# \cD$ of initial distribution $\cD$, where $T_\cD$ is a map \st $T_\cD(\bx, g, y) \eqdef (T(\bx), g, y)$.We define the family of all distributions resulted from the admissible perturbations as
% \begin{equation}
% \label{eq:D-T-rho}
% \cD_{\cT_\rho} \eqdef \{({T_\cD})_\# \cD : T \in \cT_\rho\} \enspace.
% \end{equation}
% Formulation~\eqref{eq:D-T-rho} recovers the Definition~\ref{def:stability} with  $\cP = \cD_{\cT_\rho}$, $\pi \in \cD_{\cT_\rho}$ and 
% $\Delta(h; \pi) = \sup_{(\bx',g, y) \in \pi} = \Delta(h(T^{-1}(\bx')), h(\bx'))$, resulting in the following problem:
% \[
% \sup_{\cD' \in \cD_{\cT_\rho}} \sup_{(\bx',g, y) \in \cD'} \Delta(h(T^{-1}(\bx')), h(\bx')) \leq L\rho \enspace.
% \]
% \end{remark}
\subsection{Group fairness}
Proof of Proposition~\ref{prop:DP-stab}.
\begin{proposition*}[Proposition~\ref{prop:DP-stab}, DP and Stability] 
Let $h:\X\to\{0,1\}$ be a (possibly randomized) prediction function, $G\in\{0,1\}$ be a sensitive attribute, and denote $r_g (h) \eqdef \Prob(h(\bX)=1\mid G=g), \forall g\in\{0,1\}$. Let $\cD_\pi$ be a perturbed distribution according to Perturbation~\ref{pert:G}. We define the \textbf{global positive} prediction rate under $\cD_\pi$ as
$P_h(\pi) \eqdef \Prob_{\cD_\pi}(h(\bX)=1)$.
Then, the following statements are equivalent:
\begin{enumerate}
  \item $h$ satisfies Demographic Parity: $r_0 (h) = r_1 (h)$.
  \item $P_h(\pi)$ is constant in $\pi$ on $[0,1]$ (i.e., the global positive rate is invariant to any resampling of $G$).
\end{enumerate}
Moreover, one has the explicit identity
\[
P_h(\pi) = r_0(h) + \pi (r_1(h)-r_0(h)) \enspace,
\]
so that
\[
\sup_{\pi,\pi'\in[0,1]} |P_h(\pi)-P_h(\pi')| = |r_1(h)-r_0(h)| \enspace.
\]
\end{proposition*}
\begin{proof}
By the law of total probability, for any $\pi\in[0,1]$,
\[
P_h(\pi) \;=\; \Prob_{\cD_\pi}(h(\bX)=1)
       \;=\; \pi\, r_1(h) + (1-\pi)\, r_0(h) \enspace.
\]
We note that if $r_0(h) = r_1(h)$, then $P_h(\pi) = \pi r_1(h) + (1-\pi) r_1(h) = r_1(h)$, which does not depend on $\pi$, hence (1) $\Rightarrow$ (2).

Secondly, if we suppose that $P_h(\pi)$ is constant in $\pi$, then for all $\pi,\pi'\in[0,1]$, we have
\[
\pi r_1 (h) + (1-\pi) r_0(h) \;=\; \pi' r_1(h) + (1-\pi') r_0(h) \enspace, 
\]
which can be rewritten as 
\[
(\pi-\pi')(r_1(h)-r_0(h))=0 \enspace.
\]
Since we may choose $\pi\neq \pi'$, it follows that $r_1(h)-r_0(h)=0$, i.e.,
$r_0(h)=r_1(h)$, which is exactly Demographic Parity (so (2) $\Rightarrow$ (1)).

This proves the equivalence.
\end{proof}

Proof of Proposition~\ref{prop:EOpp-stab}.
\begin{proposition*}[Proposition~\ref{prop:EOpp-stab}, EO and Stability]
Let $h:\X\to\{0,1\}$ be a (possibly randomized) prediction function, $G\in\{0,1\}$ be a sensitive attribute, and denote $q_{g,y} (h) \eqdef \Prob(h(\bX)=1\mid G=g, Y=y), \forall g, y\in\{0,1\}$ . Let $\cD_\pi$ be a perturbed distribution according to Perturbation~\ref{pert:G-1}. We define the \textbf{true positive} prediction rate under $\cD_\pi$ as
$P_{h\mid Y=1}(\pi) \eqdef \Prob_{\cD_\pi}(h(\bX)=1 \mid Y=1)$. Then the following statements are equivalent:
\begin{enumerate}
    \item h satisfies Equal Opportunity: $q_{0,1}(h) = q_{1,1}(h)$.
    \item $P_{h\mid Y=1}(\pi)$ is constant in $\pi$ on $[0,1]$ (i.e., the true positive rate is invariant to any resampling of $G$ within positive label class).
\end{enumerate}
Moreover, one has the explicit identity
\[
P_{h \mid Y=1}(\pi) = q_{0,1}(h) + \pi (q_{1,1}(h) - q_{0,1}(h)) \enspace,
\]
so that
\[
\sup_{\pi,\pi'\in[0,1]} |P_{h \mid Y=1}(\pi)-P_{h \mid Y=1}(\pi')| = |q_{1,1}(h)-q_{0,1}(h)|\enspace.
\]
\end{proposition*}
\begin{proof}
Let us show that $(1\Rightarrow2)$. Denoting $q := q_{0,1}(h) = q_{1,1}(h)$ and applying the law of total probability, we get
\[
    P_{h \mid Y=1}(\pi) = \Prob_{\cD_\pi}(h(X)=1 \mid Y=1) = (1-\pi) q_{0,1}(h) + \pi q_{1,1}(h) = q \enspace,
\]
which does not depend on $\pi$.
Now let us show that $(2\Rightarrow1)$.
Assuming $P_{h \mid Y=1}(\pi)$ is constant in $\pi$, we can state that
\[
    (1-\pi)q_{0,1}(h) + \pi q_{1,1}(h) = (1-\pi') q_{0,1}(h) + \pi' q_{1,1}(h) \enspace.
\]
Rewriting the above as 
\[(\pi-\pi')(q_{1,1}(h)-q_{0,1}(h))=0 \enspace,\]
we deduce that $q_{1,1}(h)=q_{0,1}(h)$.

The proof is concluded.
\end{proof}

\section{Details and proofs for Section~\ref{sec:methodology}}
\label{app:rates} \subsection{Importance weights representation}
\label{app:weights}
In this section we show that the distribution shifts under our consideration can be described by their corresponding importance-weights, defined by the Radon-Nikodym derivative
\[
w_\pi(\bx,g,y) \eqdef {\dv{\cD_\pi}{\cD}}(\bx,g,y) \enspace.
\]
Recall, that $p_g = \Prob (G = g)\,\, \forall g \in \{0,1\}$ and $\Prob (G = g \mid Y=y)\,\, \forall g,y \in \{0,1\}\enspace.$

For Perturbation~\ref{pert:G},
considering the factorization $\Prob(\bX, G, Y) = \Prob(G)\Prob(\bX,Y \mid G)$, and computing the Radon-Nikodym derivative, we get
\[
w_\pi (g) =
\begin{cases}
  \frac{\pi}{p_1}\enspace, & \text{if } g = 1 \enspace,\\
  \frac{1-\pi}{p_0}\enspace, & \text{if } g = 0 \enspace. 
\end{cases}
\]

For Perturbation~\ref{pert:G-1},
considering the factorization $\Prob(\bX, G, Y) = \Prob(Y)\Prob(G \mid Y)\Prob(\bX \mid Y, G)$, and computing the Radon-Nikodym derivative, we get
\[
w_\pi (g, y) =
\begin{cases}
  \frac{\pi}{p_{1,1}}\enspace, & \text{if } (g,y) = (1,1) \enspace,\\
  \frac{1-\pi}{p_{0,1}}\enspace, & \text{if } (g,y) = (0,1) \enspace, \\
  1\enspace, & \text{if } y = 0 \enspace.
\end{cases}
\]

\subsection{Convergence rates}
\label{app:conv-rates}
First, let us recall the following two crucial lemmas.
\begin{lemma}\citep[Hoeffding inequality]{Vershynin18} Let $\bX_1,\cdots, \bX_N$ be independent random variables such that $\bX_i\in[a_i,b_i]$ for every $i$. Then, for any $t>0$, we have
\[
\Prob\left(\sum_{i=1}^N (\bX_i - \Exp[\bX_i]) \geq t\right) \leq \exp\left(-\frac{2t^2}{\sum_{i=1}^N (b_i - a_i)^2}\right) \enspace.
\]
\end{lemma}
\begin{lemma}\citep[Bernstein inequality]{Vershynin18}
Let $\bX_1,\cdots, \bX_N$ be independent, mean-zero random variables satisfying $|\bX_i|\leq K$ for every $i$. Let $\sigma^2 = \sum_{i=1}^N \Exp \bX_i^2$ is the variance of the sum. Then, for any $t>0$, we have
\[
\Prob\left( \left| \sum_{i=1}^N \bX_i \right|\geq t\right) \leq 2\exp\left(-\frac{\nicefrac{t^2}{2}}{\sigma^2 + \nicefrac{Kt}{3}}\right) \enspace.
\]
\end{lemma}
We now examine the different sources of error arising from the estimation of population-level quantities.
\subsubsection{Regression error rates}
\label{app:reg-rates}
Recall the following expressions
\[
\bdeta_{\balpha^\star} = \sum_{k\in[K]} \balpha^\star_k {\eta}_k\enspace,\qquad \bdeta_{\hat\balpha} = \sum_{k\in[K]} \hat\balpha_k \eta_k \enspace, \qquad \hat{\bdeta}_{\balpha^\star} = \sum_{k\in[K]} \balpha^\star_k \hat{\eta}_k\enspace,\qquad \hat{\bdeta}_{\hat\balpha} = \sum_{k\in[K]} \hat\balpha_k \hat{\eta}_k \enspace.
\]

\begin{lemma}
\label{lem:rates:reg} Let Assumption~\ref{ass:reg-rate} and Claims~\ref{claim:phi-bd}--\ref{claim:bound-w} hold. Then, for all for every
$\balpha \in \Delta_K$ and all $\pi \in [\ubar{\pi}, 1-\ubar{\pi}]$, it holds that
\[
\Big| \Phi(\hat{\bdeta}_\balpha; \pi) -\Phi(\bdeta_\balpha; \pi) \Big| \leq BW \rho_N \eqdef \delta_N \enspace.
\]
\end{lemma}
\begin{proof}
By the linearity of $\phi_h$ in $h$ and the definition of the mixture, for any $\balpha\in\Delta_K$, we have
\begin{align*}
| \Phi(\hat{\bdeta}_\balpha; \pi) -\Phi(\bdeta_\balpha; \pi) | &= \Big| \sum_{k=1}^K \alpha_k\, \Big[\Phi(\hat{\eta}_k; \pi) - \Phi(\eta_k; \pi)\Big]  \Big| \leq \max_{1 \leq k \leq K} \Big|\Phi(\hat{\eta}_k; \pi) - \Phi(\eta_k; \pi)\Big| \enspace.
\end{align*}
For each $k$ and $\pi$, from Hölder's inequality and Claims~\ref{claim:phi-bd}--\ref{claim:bound-w}, we obtain
\begin{align*}
|\Phi(\hat{\eta}_k; \pi) - \Phi(\eta_k; \pi)|
&= |\Exp[w_\pi(\bZ) \phi_{\hat{\eta}_k - \eta_k}(\bZ)]|
\leq \norm{w_\pi(\bZ)}_{L^\infty} \Exp[\phi_{\hat{\eta}_k - \eta_k}(\bZ)] \\ 
&\leq WB \Exp[|\hat{\eta}(\bZ) - \eta(\bZ)|] \enspace.
\end{align*}
Assumption~\ref{ass:reg-rate} concludes the proof.
\end{proof}

\subsubsection{Perturbation error rates}
\label{app:rates:pert}
\begin{lemma}[Stratified sampling]
\label{lem:rates:pert} Let $\Pi=U(\ubar{\pi},1-\ubar{\pi})$ and $\cI_1,\cdots,\cI_m$ be a uniform partition of $[\ubar{\pi},1-\ubar{\pi}]$ into $m$ intervals of length $|\cI| = \frac{(1-2\ubar{\pi})}{m}$. We sample $\pi'_j \sim U(\cI_j)$ independently for $j=1,\cdots,m$. Let $f:[\ubar{\pi},1-\ubar{\pi}]\to\bbR$ be a $L-$Lipschitz function. For all $a_{\delta_m}\in(0,1)$ with probability greater than $1-a_{\delta_m}$, it holds that
\[
\Big|\frac{1}{m}\sum_{j=1}^m f(\pi_j') - \Exp_\Pi[f(\pi)]\Big| \leq \frac{L(1-2\ubar{\pi})}{m^{\nicefrac{3}{2}}} \sqrt{2\log(\nicefrac{2}{a_{\delta_m}})} \eqdef \delta_{L,m} \enspace.
\]
% \[
% \Exp_{\Pi^m}\Big[
%     \Big|\frac{1}{m}\sum_{j=1}^m f(\pi_j') - \Exp_\Pi[f(\pi)]\Big|\Big] \leq \frac{L(1-2\ubar{\pi})}{m^{\nicefrac{3}{2}}} \eqdef \delta_m \enspace,
% \]
% where $\Exp_{\Pi^m}$ is the expectation over $\pi'_1,\cdots,\pi'_m.$
\end{lemma}

\begin{proof}
Denote the stratum mean by $\mu_j \eqdef \frac{1}{|\cI|}\int_{\cI_j} f(\pi)\, \d\pi$ and the intra-stratum deviation by $\xi_j \eqdef f(\pi_j') - \mu_j$. Notice, that 
$
\frac{1}{m}\sum_{j=1}^m \mu_j
  = \frac{1}{m}\sum_j \frac{1}{|\cI|}\int_{\cI_j} f(\pi)\, \d\pi
  = \frac{1}{1-2\ubar{\pi}}\int_{\ubar{\pi}}^{1-\ubar{\pi}} f(\pi)\, \d\pi
  = \Exp_\Pi[f(\pi)] \enspace,
$
therefore $\frac{1}{m}\sum_j f(\pi_j') - \Exp_\Pi[f(\pi)] = \frac{1}{m}\sum_{j=1}^m \xi_j$. Moreover, for $\pi \in \cI_j$,
  $|f(\pi) - \mu_j| \leq \sup_{\pi, \pi' \in \cI_j} |f(\pi) - f(\pi')|
  \leq L |\cI|$, so $|\xi_j| \leq L |\cI|$.
Combining this with the facts that $\Exp[\xi_j] = 0$ for each $j$, that $\xi_j$ are independent, and applying Hoeffding's inequality, we obtain 
\[
\Prob\Big(\Big|\frac{1}{m}\sum_j \xi_j\Big| > \epsilon \Big) \leq 2 \exp \Big(-\frac{m \epsilon^2}{2 L^2 |\cI|^2}\Big) = 2 \exp \Big(-\frac{m^3 \epsilon^2}{2 L^2 (1-2\ubar{\pi})^2}\Big) \enspace.
\]
Setting $a_{\delta_m} = 2 \exp (-\frac{m^3 \epsilon^2}{2 L^2 (1-2\ubar{\pi})^2})$ and obtaining $\epsilon = \frac{\sqrt{2}L(1-2 \ubar{\pi})}{m^{\nicefrac{3}{2}}} \sqrt{\log(\nicefrac{2}{a_{\delta_m}})}$ concludes the proof.
\end{proof}

\subsubsection{Sample generalization rates}
\label{app:sample-rates}
\begin{lemma}
\label{lem:rates:sample}
For all $a_{\delta_n}\in(0,1)$ with probability greater than $1-a_{\delta_n}$ and conditionally on $\cD^N$,  with probability greater than $1-a_{\delta_n}$ it holds that
\[
\max_{k\in[K],j\in[m]}
\Big|\hat{\Phi}_n(\hat{\eta}_k; \pi_j') - \Phi(\hat{\eta}_k; \pi_j')\Big| \leq \sqrt{\frac{2W^2 \log(\nicefrac{2Km}{a_{\delta_n}})}{n}} + \frac{2W\log(\nicefrac{2Km}{a_{\delta_n}})}{3n} \eqdef \delta_n \enspace.
\]
Moreover, for all $\balpha \in \Delta_K$ and $j\in[m]$, it holds that 
\[
\Big|\hat{\Phi}_n(\hat{\bdeta}_\balpha; \pi_j') - \Phi(\hat{\bdeta}_\balpha; \pi_j')\Big| \leq \delta_n \enspace.
\]

\end{lemma}
\begin{proof}
Let us fix $k\in[K]$ and $j\in[m]$. The function $\hat{\eta}_k$ is fixed conditionally on sample $\cD^N$. Let us define $\xi_i \eqdef w_{\pi'_j}(\bZ_i)\phi_{\hat{\eta}_k}(\bZ_i)$, for $\bZ_i \simiid \cD^n$. We have $\Exp[\xi_i] = \Phi(\hat{\eta}_k; \pi_j')$, $|\xi_i|\leq W$, $|\xi_i-\Exp[\xi_i]|\leq W$ and $\Var(\xi_i)\leq\Exp[\xi_i^2]\leq W^2$. Applying Bernstein's inequality, for any $\epsilon>0$ it holds that
\[
\Prob_{\cD^n}\Big(\Big|\frac{1}{n}\sum_i \xi_i\Big| > \epsilon \enspace|\enspace \cD^N \Big) \leq 2\exp\Big(-\frac{n\epsilon^2}{2W^2 + \frac{2}{3}W\epsilon}\Big) \enspace.
\]
Setting $\frac{a_{\delta_n}}{Km} = 2\exp(-\frac{n\epsilon^2}{2W^2 + \frac{2}{3}W\epsilon})$, obtaining $\epsilon = \sqrt{\frac{2W^2 \log(\nicefrac{2Km}{a_{\delta_n}})}{n}} + \frac{2W\log(\nicefrac{2Km}{a_{\delta_n}})}{3n}$, and applying union bound over all pairings$(k,j)\in[K]\times[m]$ concludes the proof of the first inequality.

Moreover, using the triangle inequality and the fact that $\alpha_k \geq 0$, $\sum_k \alpha_k = 1$, we obtain
\begin{align*}
\Big|\hat{\Phi}_n(\hat{\bdeta}_\balpha; \pi_j') - \Phi(\hat{\bdeta}_\balpha; \pi_j')\Big|
&= \Big|\sum_{k=1}^K \alpha_k
(\hat{\Phi}_n(\hat{\eta}_k;\, \pi_j') - \Phi(\hat{\eta}_k; \pi_j'))\Big| \\
&\leq \sum_{k=1}^K \alpha_k
\Big|\hat{\Phi}_n(\hat{\eta}_k; \pi_j') - \Phi(\hat{\eta}_k; \pi_j')\Big| \\
&\leq \max_{k} |\hat{\Phi}_n(\hat{\eta}_k; \pi_j') - \Phi(\hat{\eta}_k; \pi_j')| \leq \delta_n.
\end{align*}
The proof is concluded.
\end{proof}

\subsubsection{Final constraint rates}
Formal version and the proof of the first statement of Theorem~\ref{thm:main-rates}.
\begin{theorem}[Constraint rates]
\label{thm:rates:constraint}   
Let $\hat{\balpha}$ be a simplex vector, \st $\hat{\cC}_{n,m}({\hat\bdeta}_{\hat\balpha};t) \leq \delta$. For all $a_{\delta_m}, a_{\delta_n} \in (0,1)$ and conditionally on $\cD^N$, with probability greater than $1-a_{\delta_m}-a_{\delta_n}$ it holds that
\[
    \cC_\Pi(\hat{\bdeta}_{\hat{\balpha}}; t)
    \leq \Big(\sqrt{\delta} + \delta_n \Big)^2 + \delta_m \enspace,
\]
where 
\[
\delta_m \eqdef \frac{2(2-\ubar{\pi})(1-2\ubar{\pi})}{m^{\nicefrac{3}{2}}}
  \sqrt{2\log(\nicefrac{4}{a_{\delta_m}})} \quad\text{and}\quad \delta_n = \sqrt{\frac{2W^2 \log(\nicefrac{2Km}{a_{\delta_n}})}{n}} + \frac{2W\log(\nicefrac{2Km}{a_{\delta_n}})}{3n} \enspace.
\]
\end{theorem}
\begin{proof}
Let us fix $j\in[m]$. By the triangle inequality, and Lemma~\ref{lem:rates:sample}, we have
\begin{align*}
|\Phi(\hat{\bdeta}_{\hat{\balpha}}; \pi_j') - t| \leq |\hat{\Phi}_n(\hat{\bdeta}_{\hat{\balpha}}; \pi_j') - t| + |\hat{\Phi}_n(\hat{\bdeta}_{\hat{\balpha}}; \pi_j') - \Phi(\hat{\bdeta}_{\hat{\balpha}}; \pi_j')| \leq |\hat{\Phi}_n(\hat{\bdeta}_{\hat{\balpha}}; \pi_j') - t| + \delta_n \enspace.
\end{align*}
Squaring both sides of the above inequality, we get $(\Phi(\hat{\bdeta}_{\hat{\balpha}}; \pi_j')- t)^2 \leq (|\hat{\Phi}_n(\hat{\bdeta}_{\hat{\balpha}}; \pi_j') - t| + \delta_n)^2\enspace.$ Averaging over $j=1,\cdots,m$ yields
\begin{align*}
\frac{1}{m}\sum_{j=1}^m (\Phi(\hat{\bdeta}_{\hat{\balpha}}; \pi_j') - t)^2 
&\leq \frac{1}{m}\sum_j (|\hat{\Phi}_n(\hat{\bdeta}_{\hat{\balpha}}; \pi_j') - t| + \delta_n)^2 \\
&=\hat{\cC}_{n,m}(\hat{\bdeta}_{\hat{\balpha}}; t) + 2\delta_n \frac{1}{m}\sum_j |\hat{\Phi}_n(\hat{\bdeta}_{\hat{\balpha}}; \pi_j') - t|
    + \delta_n^2 \enspace.
\end{align*}
By Cauchy--Schwarz inequality, and the fact that $\hat{\cC}_{n,m}(\hat{\bdeta}_{\hat{\balpha}};t) \leq \delta$, we have 
\[
\frac{1}{m}\sum_j |\hat{\Phi}_n(\hat{\bdeta}_{\hat{\balpha}}; \pi_j') - t| \leq \sqrt{\frac{1}{m}\sum_j (\hat{\Phi}_n(\hat{\bdeta}_{\hat{\balpha}}; \pi_j') - t)^2} = \sqrt{\hat{\cC}_{n,m}(\hat{\bdeta}_{\hat{\balpha}}; t)} \leq \sqrt{\delta} \enspace.
\]
Combining the inequalities above we get
\begin{align}
\label{eq:bd:Phi:delta-n}
\frac{1}{m}\sum_{j=1}^m (\Phi(\hat{\bdeta}_{\hat{\balpha}}; \pi_j') - t)^2 \leq (\sqrt{\delta} + \delta_n)^2 \enspace.
\end{align}

For any $\balpha\in\Delta_K$, by linearity of $\alpha(\cdot)$\footnote{Do not confuse the functional $\alpha(\cdot)$ from the decomposition of $\Phi(h;\pi)$ with the simplex vector $\balpha=(\alpha_1,\cdots,\alpha_K)$.} and $\gamma(\cdot)$, that follows from linearity of $\Phi(h;\pi)$ in Claim~\ref{claim:Phi}, and denoting $\alpha_\balpha \eqdef \sum_{k=1}^K \alpha_k \alpha(\hat\eta_k)$ and $\gamma_\balpha \eqdef \sum_{k=1}^K \alpha_k \gamma(\hat\eta_k)$, we get 
\[
(\Phi(\hat\bdeta_{\balpha};\pi) - t)^2 = (\alpha_\balpha - t)^2 + 2(\alpha_\balpha - t )\gamma_\balpha \pi + \gamma_\balpha^2 \pi^2 \enspace.
\]
Averaging over $\pi'_1,\cdots,\pi'_m$ and subtracting from $\cC_\Pi(\hat{\bdeta}_\balpha;t)$, we get
\begin{align*}
\cC_\Pi(\hat{\bdeta}_\balpha;t) &- \frac{1}{m}\sum_{j=1}^m (\Phi(\hat{\bdeta}_{{\balpha}}; \pi_j') - t)^2 \\
&= 2 (\alpha_\balpha - t ) \gamma_\balpha \Big\{\Exp_\Pi[\pi] - \frac{1}{m}\sum_{j=1}^m \pi'_j\Big\} + \gamma_\balpha^2 \Big\{\Exp_\Pi[\pi^2] - \frac{1}{m}\sum_{j=1}^m (\pi'_j)^2\Big\}
\enspace.
\end{align*}
Now, let us notice that the maps $\pi\mapsto\pi$ and $\pi\mapsto\pi^2$ are respectively $1$--Lipschitz and $2(1-\ubar{\pi})$--Lipschitz on $[\ubar{\pi},1-\ubar{\pi}]$. Applying Lemma~\ref{lem:rates:pert} with a union bound, and recalling that $\alpha_\balpha\in[0,1]$ and $\gamma_\balpha\in[-1,1]$, with probability at least $1-a_{\delta_m}$ and uniformly in $\balpha\in\Delta_K$, we obtain
\[
\cC_\Pi(\hat{\bdeta}_{{\balpha}}; t)
\leq \frac{1}{m}\sum_j (\Phi(\hat{\bdeta}_{{\balpha}}; \pi_j') - t)^2 + \frac{2(2-\ubar{\pi})(1-2\ubar{\pi})}{m^{\nicefrac{3}{2}}}
  \sqrt{2\log(\nicefrac{4}{a_{\delta_m}})} \enspace.
\]
Finally, specializing at $\balpha=\hat\balpha$ and applying \eqref{eq:bd:Phi:delta-n}, we get
\[
\cC_\Pi(\hat{\bdeta}_{\hat{\balpha}}; t)
\leq \frac{1}{m}\sum_j (\Phi(\hat{\bdeta}_{\hat{\balpha}}; \pi_j') - t)^2 + \delta_m \leq (\sqrt{\delta} + \delta_n)^2 + \delta_m \enspace.
\]
The proof is concluded.
\end{proof}

The second statement of Theorem~\ref{thm:main-rates} is stated below as a corollary and follows directly from Theorem~\ref{thm:rates:constraint}.

\begin{corollary}[Unfairness rates]
\label{cor:unf}
Under the conditions of Theorem~\ref{thm:rates:constraint}, it holds that
\[\label{eq:unf-via-C}
\unf(\hat{\bdeta}_{\hat{\balpha}})
\leq \sqrt{\frac{
    (\sqrt{\delta} + \delta_n)^2 + \delta_m}{\Var_\Pi(\pi)}}
    \leq \frac{\sqrt{\delta} + \delta_n + \sqrt{\delta_m}}{\sqrt{\Var_\Pi(\pi)}} \enspace.
\]
Thus, if we set $\delta = \delta_0^2 \Var_\Pi(\pi)$, then we will have 
\(
\unf(\hat{\bdeta}_{\hat{\balpha}}) - \delta_0
\leq \tilde\bigO \Big(\frac{1}{n^{\nicefrac{1}{2}}} + \frac{1}{m^{\nicefrac{3}{4}}}\Big) \enspace.
\)
\end{corollary}

Formal version and the proof of the third statement of Theorem~\ref{thm:main-rates}.
\begin{theorem}[Uniform Deployment Guarantee]
\label{thm:rates:deploy}
Let $\hat{\bdeta}_{\hat\balpha}$ be any feasible solution of~\eqref{pr:socp} with threshold $\delta$
and target $t \in [0,1]$.
Under the conditions of Theorem~\ref{thm:rates:constraint}, 
For all $a_{\delta_m}, a_{\delta_n} \in (0,1)$ and conditionally on $\cD^N$, with probability greater than $1-a_{\delta_m}-a_{\delta_n}$ over $(D^n, \Pi^m)$, the following holds simultaneously for
\textbf{all} $\pi' \in [\ubar{\pi},1-\ubar{\pi}]$
\begin{equation*}
    \bigl|\Phi(\hat{\bdeta}_{\hat\balpha};\pi') - t\bigr|
    \leq
    (\sqrt{\delta} + \delta_n  + \sqrt{\delta_m})
    \cdot
    \left(1 + \frac{|\pi' - \mu_\Pi|}{\sqrt{\Var_\Pi(\pi)}}\right) \enspace.
    %\label{eq:rates:deployment}
\end{equation*}
\end{theorem}

\begin{proof}
It follows from~\eqref{eq:t-var-decomp} and Theorem~\ref{thm:rates:constraint} that
\[
(\Exp_{\pi\sim\Pi}[\Phi(\hat{\bdeta}_{\hat{\balpha}};\pi)]-t)^2 \leq (\sqrt{\delta} + \delta_n )^2 + \delta_m \enspace.
\]
Taking the square root, we obtain
\[
|\Exp_{\pi\sim\Pi}[\Phi(\hat{\bdeta}_{\hat{\balpha}};\pi)]-t| \leq \sqrt{\delta} + \delta_n  + \sqrt{\delta_m} \enspace.
\]
For a fixed $\pi'$, we have
\[
\Phi(\hat{\bdeta}_{\hat\balpha};\pi') - t =
\underbrace{
\bigl[\Phi(\hat{\bdeta}_{\hat\balpha};\pi') - \Exp_\Pi[\Phi(\hat{\bdeta}_{\hat\balpha};\pi)]\bigr]
}_{\text{fluctuation term}} +
\underbrace{
\bigl[\Exp_\Pi[\Phi(\hat{\bdeta}_{\hat\balpha};\pi)] - t\bigr]}_{\text{bias term}}.
\]
Let us denote by $\mu_\Pi \eqdef \Exp_\Pi[\pi]$. For the fluctuation term, by the linearity of $\Phi(\hat{\bdeta}_{\hat\balpha};\pi)$ and the identity
$\Exp_\Pi[\alpha(h) + \pi\gamma(h)] = \alpha(h) + \mu_\Pi\gamma(h)$, we have
\begin{align*}
    \Phi(\hat{\bdeta}_{\hat\balpha};\pi') - \Exp_\Pi[\Phi(\hat{\bdeta}_{\hat\balpha};\pi)]
    &=
    \bigl(\alpha(\hat{\bdeta}_{\hat\balpha}) + \pi'\gamma(\hat{\bdeta}_{\hat\balpha})\bigr)
    -
    \bigl(\alpha(\hat{\bdeta}_{\hat\balpha}) + \mu_\Pi\gamma(\hat{\bdeta}_{\hat\balpha})\bigr) \\
    &=
    (\pi' - \mu_\Pi)\gamma(\hat{\bdeta}_{\hat\balpha}) \enspace.
\end{align*}
Therefore,
\begin{equation}
    \Phi(\hat{\bdeta}_{\hat\balpha};\pi') - t
    =
    (\pi' -\mu_\Pi)\gamma(\hat{\bdeta}_{\hat\balpha})
    + \bigl[\Exp_\Pi[\Phi(\hat{\bdeta}_{\hat\balpha};\pi)] - t\bigr] \enspace.
    \label{eq:decomp_pi}
\end{equation}

Applying the triangle inequality to~\eqref{eq:decomp_pi} and
substituting bounds from Theorem~\ref{thm:rates:constraint},
\begin{align*}
    \bigl|\Phi(\hat{\bdeta}_{\hat\balpha};\pi') - t\bigr|
    &\leq
    |\pi' - \mu_\Pi|\cdot|\gamma(\hat{\bdeta}_{\hat\balpha})| + \bigl|\Exp_\Pi[\Phi(\hat{\bdeta}_{\hat\balpha};\pi)] - t\bigr| \\
    &\leq
    |\pi' - \mu_\Pi| \cdot \frac{(\sqrt{\delta} + \delta_n  + \sqrt{\delta_m})}{\sqrt{\Var_{\Pi}(\pi)}} +
    (\sqrt{\delta} + \delta_n  + \sqrt{\delta_m}) \\
    &=
    \left(1 + \frac{|\pi' - \mu_\Pi|}{\sqrt{\Var_{\Pi}(\pi)}}\right) \cdot (\sqrt{\delta} + \delta_n  + \sqrt{\delta_m}) \enspace.
\end{align*}
Since this bound holds for all $\pi' \in [0,1]$ simultaneously,
the proof is concluded.
\end{proof}

Let us also prove the following lemma, that we will use in later analysis.
\begin{lemma}[Reverse feasibility]
\label{lem:rates:constraint-reverse}
Under the conditions of Lemma~\ref{lem:rates:reg}, Lemma~\ref{lem:rates:pert} and Lemma~\ref{lem:rates:sample}, let  $\sqrt{\delta} > \delta_n + \delta_N$ and $(\sqrt{\delta} - \delta_n - \delta_N)^2 > \delta_m$.  We define the \emph{tightened level} as 
\[\label{eq:tilde-delta}
\tilde\delta \eqdef (\sqrt{\delta} - \delta_n - \delta_N)^2 - \delta_m \enspace.
\]
If $\balpha\in\Delta_K$ satisfies $\cC_\Pi(\eta_\balpha;t)\leq\tilde\delta$, then $\hat{\cC}_{n,m}(\hat\eta_\balpha;t)\leq\delta$.
\end{lemma}
\begin{proof}
Fix any $j\in[m]$.  By the triangle inequality,
\begin{align*}
\bigl|\hat\Phi_n(\hat\bdeta_\balpha;\pi'_j) - t\bigr|
&\leq \bigl|\hat\Phi_n(\hat\bdeta_\balpha;\pi'_j) - \Phi(\hat\bdeta_\balpha;\pi'_j)\bigr| + \bigl|\Phi(\hat\bdeta_\balpha;\pi'_j) - \Phi(\bdeta_\balpha;\pi'_j)\bigr| + \bigl|\Phi(\bdeta_\balpha;\pi'_j) - t\bigr| \nonumber \\
&\leq \delta_n + \delta_N  + \bigl|\Phi(\bdeta_\balpha;\pi'_j) - t\bigr| \enspace.
\end{align*}
Squaring and averaging over $j=1,\dots,m$ gives
\begin{equation}\label{eq:hat-C-int}
\hat{\cC}_{n,m}(\hat\bdeta_\balpha;t)
= \frac{1}{m}\sum_{j=1}^m \bigl(\hat\Phi_n(\hat\bdeta_\balpha;\pi'_j)-t\bigr)^2 \leq \frac{1}{m}\sum_{j=1}^m \Bigl(\bigl|\Phi(\bdeta_\balpha;\pi'_j)-t\bigr| + \delta_n + \delta_N\Bigr)^2 \enspace.
\end{equation}

Let us define $g(\pi)\eqdef(\Phi(\bdeta_\balpha;\pi)-t)^2$. Since the map $\pi\mapsto\Phi(\bdeta_\balpha;\pi)$is affine, $g$ is a polynomial of degree two in~$\pi$, and in particular Lipschitz on $[\ubar\pi,1-\ubar\pi]$ with constant at most $2W^2$, applying Lemma~\ref{lem:rates:pert}, we obtain
\[
\frac{1}{m}\sum_{j=1}^m g(\pi'_j) \leq \Exp_\Pi[g(\pi)] + \delta_m = \cC_\Pi(\bdeta_\balpha;t) + \delta_m \leq \tilde\delta + \delta_m \enspace.
\]
Applying the Cauchy--Schwarz inequality, we get
\[
\frac{1}{m}\sum_{j=1}^m \bigl|\Phi(\bdeta_\balpha;\pi'_j)-t\bigr|
\leq \sqrt{\frac{1}{m}\sum_{j=1}^m g(\pi'_j)}
\leq \sqrt{\tilde\delta + \delta_m}\enspace.
\]
Expanding the right-hand-side of ~\eqref{eq:hat-C-int}, we obtain
\begin{align*}
\hat{\cC}_{n,m}(\hat\bdeta_\balpha;t) 
&\leq \frac{1}{m}\sum_j g(\pi'_j) + 2(\delta_n+\delta_N)\cdot\frac{1}{m}\sum_j\bigl|\Phi(\eta_\balpha;\pi'_j)-t\bigr| + (\delta_n+\delta_N)^2\\ 
&\leq (\tilde\delta+\delta_m) + 2(\delta_n+\delta_N)\sqrt{\tilde\delta+\delta_m} + (\delta_n+\delta_N)^2\\
&\leq \bigl(\sqrt{\tilde\delta+\delta_m} + \delta_n + \delta_N\bigr)^2 \enspace.
\end{align*}
Substituting $\tilde\delta = (\sqrt{\delta} - \delta_n - \delta_N)^2 - \delta_m$ gives
\[
\sqrt{\tilde\delta + \delta_m}
= \sqrt{(\sqrt{\delta} - \delta_n - \delta_N)^2}
= \sqrt{\delta} - \delta_n - \delta_N\enspace,
\]
and finally we get
\[
\hat{\cC}_{n,m}(\hat\bdeta_\balpha;t)
\leq \bigl((\sqrt{\delta} - \delta_n - \delta_N) + \delta_n + \delta_N\bigr)^2
= \delta \enspace.
\]
The proof is concluded.
\end{proof}

\subsubsection{Risk rates}
For prediction functions $h_1,\cdots,h_K$, let us recall the notation of population and empirical linearized risks
\[
 \overline{\risk}(h_\balpha)
 \eqdef
 \sum_{k=1}^K \alpha_k \risk(h_k)\enspace,
 \qquad
 \hat{\overline{\risk}}_n(h_\balpha)
 \eqdef
 \sum_{k=1}^K \alpha_k \hat{\risk}_n(h_k)\enspace.
\]

Let us also define
\[
 \overline\nu(\delta')
 \eqdef \min_{\balpha\in\Delta_K}\enscond{\overline{\risk}(\bdeta_\balpha)}{\ \cC(\bdeta_\balpha; t)\leq \delta'} \enspace, \text{ for } \delta'>0\enspace.
\]

The population-optimal mixture at level $\delta$ is $\balpha^\star \eqdef \argmin_{\balpha\in\Delta_K}\{\overline{R}(\bdeta_\balpha):C_\Pi(\bdeta_\balpha;t)\leq\delta\}$, so that $\overline\nu(\delta)=\overline{R}(\bdeta_{\balpha^\star})$.

\begin{lemma}[Value continuity]
\label{lem:val-cont} Under Assumption~\ref{ass:feas}, for every $ 0 \leq \gamma \leq \delta - \cC(\bdeta_{\balpha^0}; t)\enspace,$ it holds that
\[
 0 \leq \overline\nu(\delta-\gamma) - \overline\nu(\delta)
 \leq \gamma C\enspace, \quad \text{where} \quad C \eqdef \frac{\overline{\risk}(\bdeta_{\balpha^0}) - \overline\nu(\delta)}
 {\delta - \cC(\bdeta_{\balpha^0}; t)}\enspace.
\]
\end{lemma}
\begin{proof}
Let $\balpha_{\delta}^{\star}$ be a minimizer in the definition of $\overline\nu(\delta)$. Fix $\gamma\in[0,\delta-\cC(\bdeta_{\balpha^0}; t)]$ and set
\[
 \lambda_{\gamma}
 \eqdef
 \frac{\gamma}{\delta-\cC(\bdeta_{\balpha^0}; t)}
 \in [0,1]\enspace,
 \qquad
 \balpha_{\gamma}
 \eqdef (1-\lambda_{\gamma})\balpha_{\delta}^{\star}
 + \lambda_{\gamma}\balpha^0\enspace.
\] 
The simplex $\Delta_K$ is convex, hence $\balpha_{\gamma}\in\Delta_K$. Moreover,
$\balpha\mapsto \Phi(\bdeta_{\balpha};\pi)$ is affine for every fixed $\pi$, hence $\balpha\mapsto (\Phi(\bdeta_{\balpha};\pi)-t)^2$ is convex, and taking the expectation over $\pi\sim\Pi$, again, preserves the convexity. Therefore, it holds that
\[
 \cC(\bdeta_{\balpha_{\gamma}};t)\leq (1-\lambda_{\gamma})\cC(\bdeta_{\balpha_{\delta}^{\star}};t)
 + \lambda_{\gamma}\cC(\bdeta_{\balpha^0};t) \leq
 (1-\lambda_{\gamma})\delta + \lambda_{\gamma}\cC(\bdeta_{\balpha^0};t) = \delta-\gamma \enspace.
\]
Hence, $\balpha_{\gamma}$ is feasible for level $\delta-\gamma$, and thus
\[
 \overline\nu(\delta-\gamma) \leq \overline{\risk}(\bdeta_{\balpha_{\gamma}})\enspace.
\]
Using the linearity of $\overline{\risk}$, we obtain
\[
 \overline{\risk}(\bdeta_{\balpha_{\gamma}}) =
 (1-\lambda_{\gamma})\overline{\risk}(\bdeta_{\balpha_{\delta}^{\star}})
 + \lambda_{\gamma}\overline{\risk}(\bdeta_{\balpha^0}) =
 \overline\nu(\delta) + \lambda_{\gamma}(\overline{\risk}(\bdeta_{\balpha^0}) - \overline\nu(\delta))\enspace.
\]
Combining the last two, we obtain 
\[
 \overline\nu(\delta-\gamma) - \overline\nu(\delta)
 \leq  \gamma \frac{\overline{\risk}(\bdeta_{\balpha^0}) - \overline\nu(\delta)} {\delta - \cC(\bdeta_{\balpha^0}; t)} \enspace.
\]
We conclude the proof by noticing that $\overline\nu(\delta)\leq \overline\nu(\delta-\gamma)$ due to monotonicity of the feasible sets.
\end{proof}

\begin{lemma}\label{lem:risk:reg-rates}
Under Assumption~\ref{ass:reg-rate} and Assumption~\ref{ass:loss:bd}, for every $\alpha \in \Delta_K$,
\[
\big|\overline\risk(\hat{\bdeta}_\balpha) - \overline\risk(\bdeta_\balpha)\big| \leq L\rho_N \enspace.
\]
\end{lemma}
\begin{proof}
For each $k\in[K]$, by the Lipschitz property of $\ell$ in its first argument,
\[
\big|\risk(\hat\eta_k) - \risk(\eta_k)\big|
= \big|\Exp[\ell( \hat\eta_k(\bX),Y) - \ell( \eta_k(\bX),Y)]\big|
\leq  L \Exp{|\hat\eta_k(\bX) - \eta_k(\bX)|} \enspace.
\]
Since 
%$\hat{\bdeta}_\alpha - \bdeta_\balpha = \sum_k \alpha_k(\hat{\eta}_k - \eta_k)$ and 
$\balpha \in \Delta_K$,
\begin{align*}
\big|\overline\risk(\hat{\bdeta}_\balpha) - \overline\risk(\bdeta_\balpha)\big| &= \big|\sum_{k=1}^K \alpha_k (\risk(\hat\eta_k) - \risk(\eta_k)) \big| \leq \sum_{k=1}^K \alpha_k \big|\risk(\hat\eta_k) - \risk(\eta_k) \big| \\
&\leq L \max_k \Exp{|\hat{\eta}_k(\bX) - \eta_k(\bX)|} \leq L\rho_N \enspace.
\end{align*}
% \[
% \Exp{|\hat{\bdeta}_\balpha(\bX) - \bdeta_\balpha(\bX)|}
% \leq \sum_k \alpha_k \Exp{|\hat{\eta}_k(\bX) - \eta_k(X)|}
% \leq \max_k \Exp{|\hat{\eta}_k(\bX) - \eta_k(\bX)|}
% \leq \rho_N \enspace.
% \]
The proof is concluded.
\end{proof}

\begin{lemma}[Uniform linearized risk deviation]
\label{lem:lin-risk-dev}
For every $a_{\delta_R}\in(0,1)$, with probability at least $1-a_{\delta_R}$,
\[
 \sup_{\balpha\in\Delta_K}
 \bigl|\overline{\risk}(\bdeta_\balpha)-\hat{\overline{\risk}}_n(\bdeta_\balpha)\bigr|
 \leq \sqrt{\frac{\log(\nicefrac{2K}{a_{\delta_R}})}{2n}} \eqdef
 \delta_{n(R)} \enspace.
\]
Moreover, the statement is identical for estimated mixture $\hat\bdeta_\balpha$.
\end{lemma}

\begin{proof}
For each fixed $k\in[K]$, the random variables
$\ell(h_k(X_i),Y_i)$ are \iid and take values in $[0,1]$. Hoeffding's
inequality yields
\[
 \Prob_{\cD^n}\Bigl(\bigl|\risk(\eta_k)-\hat{\risk}_n(\eta_k)\bigr|>\epsilon \enspace | \enspace \cD^N \Bigr)
 \leq 2\exp(-2n\epsilon^2) \enspace.
\]
Taking the union bound over $k\in[K]$ gives
\[
 \Prob_{\cD^n}\Bigl(\max_{k\in[K]}
 \bigl|\risk(\eta_k)-\hat{\risk}_n(\eta_k)\bigr|>\epsilon \enspace | \enspace \cD^N \Bigr)
 \leq 2K\exp(-2n\epsilon^2) \enspace.
\]
Choosing $\epsilon=\sqrt{\frac{\log(\nicefrac{2K}{a_{\delta_R}})}{2n}}$ yields, with probability at least $1-a_{\delta_R}$ and conditionally on $\cD^N$,
\[
 \max_{k\in[K]}
 \bigl|\risk(\eta_k)-\hat{\risk}_n(\eta_k)\bigr| \leq \delta_{n(R)} \enspace.
\] 
Finally, for any $\balpha\in\Delta_K$,
\[
 \bigl|\overline{\risk}(\bdeta_\balpha)-\hat{\overline{\risk}}_n(\bdeta_\balpha)\bigr| =
 \Bigl|\sum_{k=1}^K \alpha_k\bigl(\risk(\eta_k)-\hat{\risk}_n(\eta_k)\bigr)\Bigr| \leq
 \sum_{k=1}^K \alpha_k\bigl|\risk(\eta_k)-\hat{\risk}_n(\eta_k)\bigr|
 \leq \delta_{n(R)} \enspace.
\]
The proof is concluded.
\end{proof}

Formal version and the proof of the fourth statement of Theorem~\ref{thm:main-rates}.
\begin{theorem}[Risk rates]
\label{thm:rates:risk}
Let $\hat\balpha$ be a solution of~\eqref{pr:socp}. Suppose Assumptions~\ref{ass:reg-rate}--\ref{ass:bound-pi} hold, and $n,m,N$ are large enough so that $\sqrt{\delta} > \delta_n+\delta_N$ and $(\sqrt{\delta} - \delta_n - \delta_N)^2 > \delta_m$. Then, for all $a_{\delta_n},a_{\delta_m},a_{\delta_{n(R)}}\in(0,1)$, conditionally on $D^N$, with probability at least $1-a_{\delta_n}-a_{\delta_m}-a_{\delta_{n(R)}}$ over $(D^n,\Pi^m)$, it holds that
\[
{\overline{\risk}} (\hat\bdeta_{\hat\balpha}) - {\overline{\risk}} (\bdeta_{\balpha^\star}) \leq L\cdot\rho_N  + C \cdot (2\sqrt{\delta}(\delta_n + \delta_N) + (\delta_n + \delta_N)^2 + \delta_m) + 2\delta_{n(R)}\enspace,
\]
where $\delta_n$ and $\delta_m$ are from Theorem~\ref{thm:rates:constraint}, and
\[
  \delta_{n(R)} = \sqrt{\frac{\log(\nicefrac{2K}{a_{\delta_R}})}{2n}} \enspace, \qquad  C = \frac{\overline{\risk}(\bdeta_{\balpha^0}) - \overline\nu(\delta)} {\delta - \cC(\bdeta_{\balpha^0}; t)} \enspace.
\]
\end{theorem}
\begin{proof}
 Let $\tilde\delta = (\sqrt{\delta} - \delta_n - \delta_N)^2 - \delta_m > 0$ be the tightened threshold from the reverse feasibility (Lemma~\ref{lem:rates:constraint-reverse}). When $n,m,N$ are ``large enough'', then $\delta_n, \delta_m, \delta_N$ are ``sufficiently'' small, so it holds that $\cC_\Pi (\bdeta_{\balpha_0}; t) < \tilde\delta$, that is the feasible set at the level $\tilde\delta$ is non-empty.
 Let us denote by $\tilde\balpha$ the minimizer of $\overline\nu(\tilde\delta)$. We have $\cC_\Pi (\bdeta_{\tilde{\balpha}}; t)$ and $\overline\nu(\tilde\delta) = \overline{\risk}(\bdeta_{\tilde\balpha})$. From Lemma~\ref{lem:rates:constraint-reverse}, we have that $\cC_{n,m} (\hat\bdeta_{\tilde\balpha}) \leq \delta$. Since $\hat\balpha$ is the empirical minimizer of~\eqref{pr:socp}, and $\tilde\balpha$ is empirically feasible, we have $\hat{\overline{\risk}}_n (\hat\bdeta_{\hat\balpha}) \leq \hat{\overline{\risk}}_n (\hat\bdeta_{\tilde\balpha})$. Combining this, with Lemma~\ref{lem:lin-risk-dev}, we get
 \[
 {\overline{\risk}} (\hat\bdeta_{\hat\balpha}) \leq \hat{\overline{\risk}}_n (\hat\bdeta_{\hat\balpha}) + \delta_{n(R)} \leq \hat{\overline{\risk}}_n (\hat\bdeta_{\tilde\balpha}) + \delta_{n(R)} \leq {\overline{\risk}} (\hat\bdeta_{\tilde\balpha}) + 2\delta_{n(R)} \enspace.
 \]
Next, we apply Lemma~\ref{lem:risk:reg-rates} and obtain 
\[
{\overline{\risk}} (\hat\bdeta_{\tilde\balpha}) \leq {\overline{\risk}} (\bdeta_{\tilde\balpha}) + L\rho_N = \overline\nu(\tilde\delta) +L\rho_N \enspace.
\]
Finally, we are in position to apply value continuity (Lemma~\ref{lem:val-cont}). Setting
\begin{align*}
 \gamma 
 &= \delta - \tilde\delta = \delta - (\sqrt{\delta} - \delta_n - \delta_N)^2 + \delta_m  \\
 &= 2\sqrt{\delta}(\delta_n + \delta_N) + (\delta_n + \delta_N)^2 + \delta_m \enspace,
\end{align*}
we get
\[
\overline{\nu}(\tilde{\delta}) = \overline{\nu}(\delta - \gamma) \leq \overline{\nu}(\delta) + \gamma C \enspace.
\]
Combining everything above, we finally obtain
\[
{\overline{\risk}} (\hat\bdeta_{\hat\balpha}) \leq {\overline{\risk}} (\bdeta_{\balpha^\star}) + \gamma C +L\rho_N + 2\delta_{n(R)}.
\]
 The proof is concluded.
 \end{proof}

\section{Extension to Equalized Odds}
\label{app:eodds}
\begin{definition}[Equalized Odds]
A classifier $h$ satisfies the EOdds property if it achieves equal true positive rates and equal false positive rates across groups:
\[
\Pr(h(\bX)=1 \mid Y=y, G=0) \;=\; \Pr(h(\bX)=1 \mid Y=y, G=1)\enspace,
\qquad \forall y\in\{0,1\}\enspace.
\]
\end{definition}
To connect EOdds to stability, we consider the following perturbation.
\begin{perturbation}
\label{pert:G-Y} The class $\cP^{G\mid Y}_{\bpi}$, with $\bpi=(\pi_0,\pi_1)$, is composed of all operators $\bP:\mathcal{M}_1(\Z)\to \mathcal{M}_1(\Z)$ such that, for all distribution $\mathcal{D}\in \mathcal{M}_1(\Z)$, 1) $\bX$'s conditional distribution given $(Y,G)$ and $Y$'s distribution are the same under $\mathcal{D}$ and under $\bP(\mathcal{D})$, 2) $\Prob_{\bP(\cD)}(G=1\mid Y=y)=\pi_y\enspace,\forall y \in \{0,1\}$.
\end{perturbation}
% \begin{perturbation}
% %\label{pert:G-Y}
% The family
% $\cD^{G\mid\Y}$ is obtained from $\cD$ by resampling $G$ so that $\Prob_{\cD_\pi}(G=1 \mid Y=y)=\pi_y, \forall y \in \{0,1\}$,
% while keeping $\Prob(\bX\mid Y,G)$ and $\Prob(Y)$ unchanged. 
% \end{perturbation}
\begin{proposition}[EOdds and Stability] 
Let $h:\X\to\{0,1\}$ be a (possibly randomized) prediction function, $G\in\{0,1\}$ be a sensitive attribute, and denote $q_{g,y} (h) \eqdef \Prob(h(\bX)=1\mid G=g, Y=y), \forall g, y\in\{0,1\}$ . Let $\cD_\bpi$ be a perturbed distribution according to Perturbation~\ref{pert:G-Y}. We define the \textbf{true positive} and \textbf{false positive} prediction rates under $\cD_\bpi$ respectively as
$P_{h\mid Y=1}(\bpi) \eqdef \Prob_{\cD_\bpi}(h(\bX)=1 \mid Y=1)$ and $P_{h\mid Y=0}(\bpi) \eqdef \Prob_{\cD_\bpi}(h(\bX)=1 \mid Y=0)$. Then, the following statements are equivalent:
\begin{enumerate}
    \item h satisfies Equalized Odds: $q_{0,y}(h) = q_{1,y}(h)\enspace, \forall y\in\{0,1\}$.
    \item $P_{h\mid Y=y}(\bpi)$ is constant in $\pi_y$ on $[0,1]\enspace, \forall y\in\{0,1\}$ (i.e., the true positive and false negative rates are invariant to any intra-class resampling of $G$).
\end{enumerate}
Moreover, one has the explicit identity
\[
P_{h \mid Y=y}(\bpi) = q_{0,y}(h) +  \pi_y (q_{1,y}(h) - q_{0,y}(h))\enspace,\quad \forall y \in \{0,1\} \enspace,
\]
so that
\[
\sup_{\pi_y,\pi_y'\in[0,1]} |P_{h \mid Y=y}(\pi_y)-P_{h \mid Y=y}(\pi_y')| = |q_{1,y}(h)-q_{0,y}(h)| \enspace,\quad \forall y \in \{0,1\} \enspace.
\]
\end{proposition}
\begin{proof}
Let us show that $(1\Rightarrow2)$. Denoting $q_y:=q_{0,y}(h) = q_{1,y}(h), \forall y \in\{0,1\}$, and applying the law of total probability, we get
\[
    P_{h \mid Y=y}(\bpi) = \Prob_{\cD_\bpi}(h(\bX)=1 \mid Y=y) = (1-\pi_y)q_{0,y}(h) + \pi_y q_{1,y}(h) = q_y \enspace, \forall y \in \{0,1\} \enspace,
\]
which does not depend on $\pi_y$.

Now let us show that $(2\Rightarrow1)$.
Assuming $P_{h \mid Y=y}(\bpi)$ is constant in $\pi_y\enspace, \forall y \in\{0,1\}$, we can state that for $y\in\{0,1\}$,it holds that
\[
    (1-\pi_y) q_{0,y}(h) + \pi_y q_{1,y}(h) = (1-\pi_y')q_{0,y}(h) + \pi_y'q_{1,y}(h) \quad \forall y \in \{0,1\} \enspace.
\]
Rewriting the above as 
\[
(\pi_y-\pi_y')(q_{1,y}(h)-q_{0,1}(h))=0 \enspace,
\]
we deduce that $q_{0,y}(h)=q_{1,y}(h)\enspace, \forall y\in\{0,1\}$.
The proof is concluded.
\end{proof}

\paragraph{Importance-weights representation} Under Perturbation~\ref{pert:G-Y}, considering the factorization $\Prob(\bX, G, Y) = \Prob(Y)\Prob(G \mid Y)\Prob(\bX \mid Y, G)$, and computing the Radon-Nikodym derivative, we get
\[
w_\bpi (g, y) =
\begin{cases}
  \frac{\pi_1}{p_{1,1}}\enspace, & \text{if } (g,y) = (1,1) \enspace,\\
  \frac{1-\pi_1}{p_{0,1}}\enspace, & \text{if } (g,y) = (0,1) \enspace, \\
  \frac{\pi_0}{p_{1,0}}\enspace, & \text{if } (g,y) = (1,0) \enspace,\\
  \frac{1-\pi_0}{p_{0,0}}\enspace, & \text{if } (g,y) = (0,0) \enspace.
\end{cases}
\]
\section{Experiments: additional details and results}
\label{app:experiments}
All experiments are run on an Apple M2 with 16\,GB RAM. The running time of one fit is reported in Table\ref{tab:runtime}. In all experiments, we operate in the \emph{unaware} setting: the sensitive attribute $G$ is not available at prediction time. 
%%% algo
\iffalse
\begin{algorithm}[t]
\caption{Pseudo-Algorithm}
\label{alg:main-pseudo}
\begin{algorithmic}[1]
\REQUIRE Samples $\{\bZ_i\}_{i=1}^n\simiid \cD$, perturbation parameter law $\Pi$, target $t$, threshold $\delta$.
\STATE Sample $\pi_1,\cdots,\pi_K \simiid \Pi$ 
\FOR{$k=1$ \textbf{to} $K$}
  \STATE Train $\hat\eta_k$ by weighted ERM on $\{\bZ_i\}$ with weights $w_{\pi_k}(\bZ_i)$.
\ENDFOR
\STATE Sample $\pi'_1,\ldots,\pi'_m\simiid \Pi$.
\STATE Build $\hat{\bA}$, with $\hat{A}_{j,k}=\hat\Phi_n(\hat\eta_k,\pi'_j)=\frac{1}{n}\sum_{i=1}^n w_{\pi'_j}(Z_i)\phi_{\hat\eta_k}(Z_i)$; and $r_k = \hat\risk_n(\hat\eta_k)$
\STATE Solve the SOCP \eqref{pr:socp} to obtain $\hat\balpha$.
\STATE Set $\hat\bdeta_{\hat\balpha}=\sum_{k=1}^K \hat\alpha_k \hat\eta_k$.
\end{algorithmic}
\end{algorithm}
\fi 

\subsection{Evaluation protocol}
\label{app:eval}

\begin{table}[t]
\centering
\caption{%
    Running time of one fit (training and prediction), in seconds: median [interquartile range] over the fairness budgets and seeds ($10\times10$ on Adult and COMPAS, $6\times5$ on ACS Income), excluding
    Shifty runs without a solution; for STABLE$_{\text{DRO}}$, dictionary and
    one value of $\lambda$, over the seeds.
    Shifty uses Setting~2 (unknown demographic shift, $\alpha=0.25$).
}
\label{tab:runtime}
\resizebox{\textwidth}{!}{
\begin{tabular}{lcccccccccc}
\toprule
 & \multicolumn{2}{c}{ERM + constraint} & \multicolumn{2}{c}{STABLE (ours)}
 & \multicolumn{2}{c}{STABLE$_{\text{DRO}}$} & \multicolumn{2}{c}{PostProc}
 & \multicolumn{2}{c}{Shifty} \\
\cmidrule(lr){2-3}\cmidrule(lr){4-5}\cmidrule(lr){6-7}\cmidrule(lr){8-9}\cmidrule(lr){10-11}
 & DP & EO & DP & EO & DP & EO & DP & EO & DP & EO \\
\midrule
Adult  & $5.86$ {\scriptsize[4.80, 7.03]} & $6.34$ {\scriptsize[5.44, 7.01]}
       & $10.4$ {\scriptsize[8.9, 12.8]}  & $11.3$ {\scriptsize[10.4, 11.9]}
       & $11.8$ {\scriptsize[9.9, 13.1]}  & $11.5$ {\scriptsize[9.7, 12.9]}
       & $1.70$ {\scriptsize[1.61, 1.83]} & $1.89$ {\scriptsize[1.83, 1.96]}
       & $178$ {\scriptsize[112, 243]}    & $160$ {\scriptsize[112, 220]} \\
COMPAS & $0.91$ {\scriptsize[0.84, 1.04]} & $1.02$ {\scriptsize[0.93, 1.19]}
       & $1.97$ {\scriptsize[1.59, 2.51]} & $2.68$ {\scriptsize[2.34, 2.96]}
       & $2.63$ {\scriptsize[2.13, 3.19]} & $3.02$ {\scriptsize[2.84, 3.15]}
       & $1.40$ {\scriptsize[1.37, 1.44]} & $1.54$ {\scriptsize[1.49, 1.63]}
       & $8.07$ {\scriptsize[4.84, 13.8]} & $7.63$ {\scriptsize[4.16, 10.9]} \\
ACS    & $116$ {\scriptsize[106, 128]}    & $101$ {\scriptsize[88, 123]}
       & $248$ {\scriptsize[233, 263]}    & $218$ {\scriptsize[209, 234]}
       & $254$ {\scriptsize[224, 263]}    & $220$ {\scriptsize[214, 255]}
       & $8$ {\scriptsize[7, 9]}          & $17$ {\scriptsize[16, 19]}
       & $39$ {\scriptsize[26, 57]}       & $24$ {\scriptsize[16, 35]} \\
\bottomrule
\end{tabular}
}
\end{table}
\paragraph{Simulating deployment shift.}
For the Adult Income and COMPAS datasets we simulate the deployment under demographic shift without retraining.

Given a fixed trained model $h$ and a test set $\{(x_i, g_i, y_i)\}$, the functional $\Phi(h;\pi')$ is estimated via importance weighting: each observation receives weight
$w_{\pi'}(g_i) = \frac{\pi'}{p_1}\mathbf{1}\{g_i=1\}
+ \frac{1-\pi'}{p_0}\mathbf{1}\{g_i=0\}$,
where $p_g = \hat{\mathbb{P}}(G=g)$ is estimated on the test set.

This reweights the empirical distribution to match a target group proportion $\pi'$ without requiring a new dataset, and corresponds exactly to Perturbation~1 of Section~\ref{sec:equiv}.
The deploy gap is then $\sup_{\pi'\in\{0,1\}}|\Phi(h;\pi')-t|$,
evaluated at the two extremes.

\paragraph{DRO-objective ablation.}\label{app:dro-ablation}
To isolate the contribution of the STABLE objective, we replace the SOCP by a distributionally robust objective while keeping everything else fixed: the same dictionary $\{\hat\eta_k\}_{k=1}^K$, the same perturbation family and points $\pi'_1,\ldots,\pi'_m$ (Perturbation~1 for DP, Perturbation~2 for EO), and the same values of $K$, $m$ and $\bar\pi$. We denote this variant STABLE$_{\text{DRO}}$. Following SA-DRO~\citep{lei2024on},
the mixture weights solve
\[
  \min_{\balpha\in\Delta_K}\;
  \max_{j\in[m]}\Big[(1-\lambda)\sum_{k=1}^K \alpha_k \hat R_{j,k}
  - \rho\,(\pi'_j-\hat p_1)^2\Big]
  + \lambda\,\Big|\sum_{k=1}^K \alpha_k\,\hat\gamma_k\Big|,
\]
where $\hat R_{j,k}=\frac1n\sum_{i=1}^n w_{\pi'_j}(Z_i)\,\ell(\hat\eta_k(X_i),Y_i)$ is the importance-weighted log-loss of the $k$-th predictor at $\pi'_j$, $\hat p_1=\hat{\mathbb{P}}(G=1)$ on the training set, and $\hat\gamma_k$ is the signed group gap of $\hat\eta_k$ (difference of mean predictions between groups for DP, restricted to $Y=1$ for EO). Since $\hat\gamma$ is linear in $\balpha$, the fairness term is computed in closed form, and the problem is a linear program, solved with CLARABEL via \texttt{cvxpy}. We fix $\rho=0.9$, the default of \citet{lei2024on}, and sweep the fairness weight $\lambda\in[0,1)$. Because $\lambda$ has no correspondence with the fairness budget of the other methods, it is swept on its own grid, chosen per dataset and constraint to cover the transition from the ERM solution to the fair regime: $\{0, 0.3, 0.35, 0.4, 0.45, 0.5, 0.7\}$ on Adult (DP), $\{0, 0.02, 0.05, 0.1, 0.15, 0.2, 0.3\}$ on Adult (EO), $\{0, 0.1, 0.3, 0.5, 0.7, 0.85, 0.95, 0.99\}$ on COMPAS, and $\{0, 0.05, 0.1, 0.2, 0.3, 0.35, 0.4, 0.45, 0.5, 0.6, 0.7, 0.85, 0.95, 0.99\}$ on ACS Income. This variant is an ablation of our method rather than a reproduction of SA-DRO, which trains a neural network, relies on kernel smoothing of the fairness term, and is defined for DP only.

\paragraph{Hyperparameter selection.}\label{app:hyperparameter}
\begin{table}[t]
\centering
\caption{Hyperparameters of STABLE. $\bar\pi$ and the dictionary are selected
on validation data; $K=30$, $K_c=10$ and $m=60$ are fixed across all settings.
The budget is set through $\delta_0$, with $\delta=\delta_0^2\,\mathrm{Var}_\Pi$.}
\label{tab:hyperparameters}
\small
\begin{tabular}{lcccccc}
\toprule
 & \multicolumn{2}{c}{\textbf{Adult}}
 & \multicolumn{2}{c}{\textbf{COMPAS}}
 & \multicolumn{2}{c}{\textbf{ACS Income}} \\
\cmidrule(lr){2-3}\cmidrule(lr){4-5}\cmidrule(lr){6-7}
 & DP & EO & DP & EO & DP & EO \\
\midrule
$\bar\pi$ & $0.02$ & $0.15$ & $0.02$ & $0.02$ & $0.05$ & $0.15$ \\
Label-flipped components & \cmark & \cmark & \xmark & \cmark & \cmark & \cmark \\$\delta_0$, frontiers
  & \multicolumn{4}{c}{10 values in $[0.005, 0.15]$}
  & \multicolumn{2}{c}{6 values in $[0.005, 0.15]$} \\
$\delta_0$, fixed budget
  & \multicolumn{4}{c}{$0.005$}
  & \multicolumn{2}{c}{$0.01$} \\
\bottomrule
\end{tabular}
\end{table}
STABLE learns a convex mixture of $K$ weighted-ERM predictors, each being an $\ell_2$-regularised logistic regression (regularisation $C=1$, \texttt{lbfgs} solver), optionally complemented by $2K_c$ label-flipped components ($K_c=10$), and the two constant
predictors (see \emph{Dictionary enrichment} below). The mixture is optimised via a SOCP evaluated on $m$ stratified perturbation points drawn from $\Pi$. We fix $K=30$ (number of mixture components) and $m=60$ (number of SOCP constraint evaluation points) across all experiments, as preliminary runs showed no sensitivity beyond these values.

The hyperparameters of STABLE are $\bar\pi$, the lower bound of the perturbation interval $\Pi = [\bar\pi,\, 1-\bar\pi]$, $\delta$, the SOCP feasibility threshold controlling the trade-off between fairness and accuracy, and the use of the label-flipped components.
$\bar\pi$ is selected via grid search over $(\bar\pi,\delta)$ on a fixed held-out validation split (seed~0 for Adult and COMPAS, three validation states for ACS Income), optimising the deploy gap
$\max(|\Phi(h;0)-t|,\,|\Phi(h;1)-t|)$.
The search grids are:
$
    \bar\pi \in \{0.02,\,0.05,\,0.10,\,0.15,\,0.20,\,0.30,\,0.40\},
    \quad
    \delta  \in \{10^{-5},\,5\!\times\!10^{-5},\,10^{-4},\,5\!\times\!10^{-4},
                  10^{-3},\,5\!\times\!10^{-3},\,10^{-2},\,5\!\times\!10^{-2}\}
$.
$\delta$ is then not fixed but swept: we set $\delta=\delta_0^2\,\mathrm{Var}_\Pi$, with $\mathrm{Var}_\Pi=(1-2\bar\pi)^2/12$, and sweep $\delta_0$ over a grid shared with the baselines, $\delta_0\in\{0.005, 0.0075, 0.01, 0.015, 0.02, 0.03, 0.05, 0.075, 0.1, 0.15\}$ on Adult and COMPAS and $\delta_0\in\{0.005, 0.01, 0.02, 0.04, 0.08, 0.15\}$ on ACS Income.
The two dictionaries, with and without the label-flipped components, reach similar deploy gaps; on the same validation split, we keep the one with the higher mean accuracy over this $\delta_0$ grid, and the dictionary with label-flipped components when the difference is below $0.005$.
All reported results use the selected $\bar\pi$ and dictionary and are averaged over 10 independent random seeds for Adult and COMPAS and 5 for ACS Income (evaluated on independent test states). The selected values are reported in Table~\ref{tab:hyperparameters}. We also report the sensitivity analysis for the grid on Adult with DP constraints in Figure \ref{fig:heatmaps}; the other settings show similar trends. All baselines are swept on the same $\delta_0$ grid. For Shifty, the fairness tolerance is $\varepsilon=\delta_0$, and seeds for which Shifty returns No-Solution-Found are discarded (up to 5 of the 10 seeds at a given budget on Adult and COMPAS, and one run out of 60 on ACS Income); all runs used the unknown-shift setting (\texttt{setting}$=2$) with $\alpha = 0.25$. ERM and all base classifiers use $\ell_2$-regularised logistic regression with $C = 1$ and the \texttt{lbfgs} solver. ERM+DP and ERM+EO (Exponentiated Gradient~\citep{pmlr-v80-agarwal18a}, with the demographic-parity and true-positive-rate-parity constraints of \texttt{fairlearn}) use $\varepsilon = \delta_0$ as the fairness tolerance. For PostProc, the constraint slack is $\varepsilon=\delta_0$, the inverse temperature of the entropic regularisation is set to $\tfrac12\sqrt{N}\log_2 N$, where $N$ is the size of the unlabeled calibration set (half of the training set), and the number of stochastic optimization steps is $T=20{,}000$ ($25{,}000$ on ACS Income). STABLE$_{\text{DRO}}$ uses the same dictionary as STABLE and is swept over its own grid of $\lambda$ (see \emph{DRO-objective ablation} above).

\begin{figure}[!t]
  \centering
  \hfill
  \begin{subfigure}[b]{0.32\linewidth}
    \centering
\includegraphics[width=\linewidth]{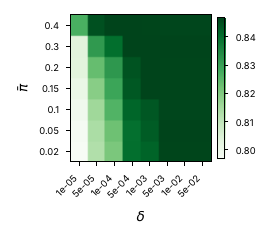}
    \caption{Accuracy$\uparrow$}
    \label{fig:dp-acc}
  \end{subfigure}
  \hfill
  \begin{subfigure}[b]{0.32\linewidth}
\centering\includegraphics[width=\linewidth]{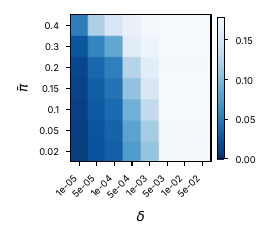}
\caption{$\unf(h)\downarrow$}
    \label{fig:dp-unf}
  \end{subfigure}
   \begin{subfigure}[b]{0.32\linewidth}
\centering\includegraphics[width=\linewidth]{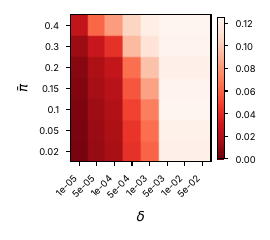}
    \caption{Deployment Gap $\downarrow$}
    \label{fig:dp-dep}
  \end{subfigure}
  \caption{Hyperparameter sensitivity on the Adult dataset (DP constraint, seed~0).
Each cell reports the value of  accuracy (left), $\unf(h)$ (center) and deploy gap $\max(|\Phi(h;0)-t|,|\Phi(h;1)-t|)$ (right)
as a function of threshold $\delta$ and the perturbation lower bound
$\bar\pi$. }
  \label{fig:heatmaps}
\end{figure}
\begin{figure}[t]
    \centering
    \hfill
    \includegraphics[width=0.95\linewidth]{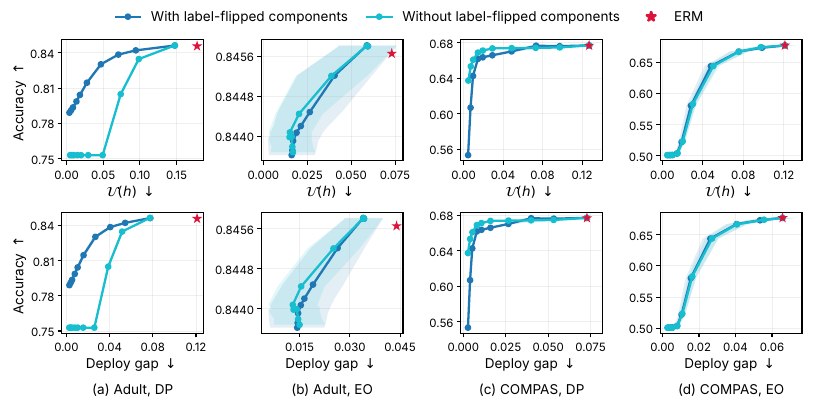}
    \caption{Dictionary ablation: STABLE with and without the label-flipped components, on the same budget grid. Accuracy against $\mathcal{U}(h)$ (top) and against the deployment gap (bottom); mean over 10 seeds, shaded regions show $\pm$ one standard deviation along the horizontal axis.}
    \label{fig:ablation-cf}
\end{figure}
\paragraph{Dictionary enrichment.}\label{app:dictionary}
When the dictionary contains only weighted-ERM components, their slopes $\gamma(\hat\eta_k)$ might only span a narrow range and none is close to zero. A small slope can then only be reached by putting weight on the constant predictors added for feasibility, which pulls the mixture toward a constant classifier. Adding $2K_c=20$ components trained on label-flipped data ($K_c=10$ flipping rates, applied to each group) provides the mixture with a wider range of slopes. Figure~\ref{fig:ablation-cf} shows that these components remove the accuracy collapse on Adult under DP, have no visible effect under EO, and lower accuracy at the tightest budgets on COMPAS under DP, with unchanged unfairness and deployment gap in all cases. Their use is therefore selected on validation data (Appendix~\ref{app:hyperparameter}); they are kept in all settings except COMPAS under DP.

\begin{figure}[t]
  \centering
  \includegraphics[width=\linewidth]{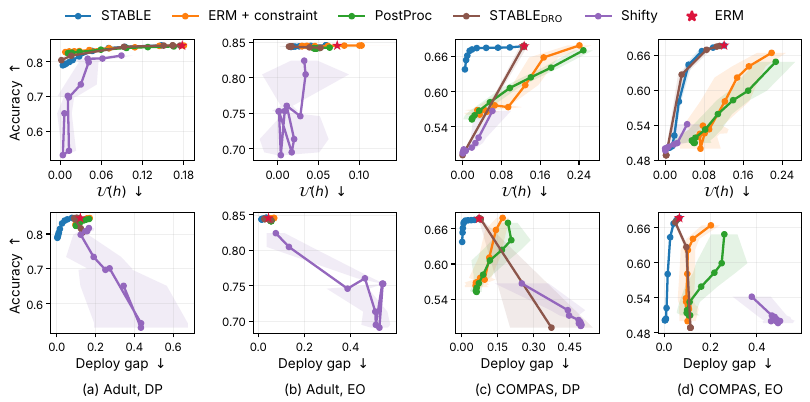}
  \caption{Pareto frontiers with Shifty. One point per fairness
  budget, mean over 10 seeds (seeds for which Shifty returns no solution are
  discarded); shaded regions show $\pm$ standard deviation along the horizontal axis.}
  \label{fig:shifty}
\end{figure}

\iffalse 
\subsection{Additional results}
\paragraph{Compas}
Figure \ref{fig:deploy-compas} presents the prediction rate of all baselines on COMPAS under different demographic shifts. These results are consistent with our previous observations on Adult Income and ACS. 
\begin{figure}[t]
    \centering  \includegraphics[width=0.8\linewidth]{figures/compas_DP_EO.pdf}
    \caption{Prediction rate $\Phi(h;\pi')$ under demographic shift on the Compas dataset for (left) DP and (right) EO constraints. Shaded regions show $\pm$ std of the curve shape over 20 random seeds.}
    \label{fig:deploy-compas}
\end{figure}
\fi 

\subsection{Additional comparison to Shifty} \label{app:sec:shifty}
Figure~\ref{fig:shifty} adds Shifty to the frontiers of Figure~\ref{fig:all-pareto}. Shifty reaches low values of $\mathcal{U}(h)$ in all settings, but at a large cost in accuracy: on COMPAS its accuracy stays close to $0.5$, and on Adult it drops by up to $0.3$ with respect to ERM. Its deployment gap, dominated by the offset of its level from $t$, increases as the budget tightens and is at least as large as that of ERM in every setting. Results on ACS Income, including Shifty, are reported in Table~\ref{tab:acs} (Appendix~\ref{app:sec:acs}).

\subsection{ACS geographic experiment.} \label{app:sec:acs}

The ACS Income task~\citep{ding2021retiring} is derived from the 2018 American Community Survey and covers over 1.5 million working-age adults across all US states. We train on California ($n =80{,}000$ subsampled) and select hyperparameters on 3 validation states ($n = 30{,}000$ per state). We use 11 test states, with White resident proportions ranging from $68\%$ to $95\%$, never seen during hyperparameter selection.

On the ACS geographic experiment under EO (see Figure~\ref{fig:eo-acs}), STABLE remains the closest to $t$ across all held-out states, with the flattest slope $\pi_{\text{state}} \mapsto \Phi_{EO}(h;\pi_{\text{state}})$ among the methods of the figure. ERM+EO and PostProc stay above $t$ in every state, with a steeper slope. The slope of STABLE$_{\text{DRO}}$ is similar to the one of STABLE, however its level remains further from $t$.

Table~\ref{tab:acs} reports the results at the same budget under both DP and EO. In both settings, STABLE attains the smallest worst-state deployment gap, with an accuracy higher than or comparable to the other methods. %ERM+constraint and PostProc have a larger deployment gap than ERM and a steeper slope than STABLE. 
STABLE$_{\text{DRO}}$ reaches an unfairness comparable to STABLE but a larger deployment gap. Shifty attains the smallest $\mathcal{U}(h)$, at the cost of a large accuracy loss and the largest deployment gap. Note that a flat prediction rate alone does not imply a small deployment gap: ERM and Shifty have the smallest slopes, but their level remains further from $t$ than that of STABLE.%On the full frontiers (Figure~\ref{fig:acs-frontiers}), STABLE attains the smallest deployment gap at its tightest budgets, while ERM+EO and PostProc have a larger deployment gap than ERM, as observed on Adult and COMPAS under EO. On the accuracy--unfairness trade-off, STABLE, ERM+EO and STABLE$_{\text{DRO}}$ attain comparable accuracies, higher than ERM.
\begin{figure}
    \centering
\includegraphics[width=0.8\linewidth]{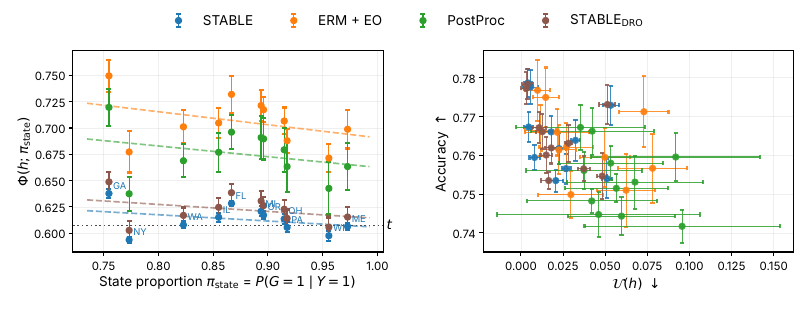}
    \caption{Geographic generalization on ACS Income under EO, at a fixed fairness budget. (left) Prediction rate $\Phi(h;\pi_{\text{state}})$ against $\pi_{\text{state}}=P(G=1\mid Y=1)$ in each held-out state; dashed lines: linear fits, dotted line: target $t$. (right) Accuracy against $\mathcal{U}(h)$. Each point is one state, mean $\pm$ std over 5 seeds.}
    \label{fig:eo-acs}
\end{figure}

\begin{table}[h]
\centering
\caption{ACS Income, held-out states, at the fixed budget of Figures~\ref{fig:acs}
and~\ref{fig:eo-acs} ($\delta_0=0.01$; STABLE$_{\text{DRO}}$ at its least unfair
$\lambda$): mean accuracy and mean $\mathcal{U}(h)$ over the states, worst-state
deployment gap, and absolute slope of $\pi_{\text{state}}\mapsto\Phi$. Mean $\pm$
std over 5 seeds.}
\label{tab:acs}
\resizebox{\textwidth}{!}{
\begin{tabular}{lcccccccc}
\toprule
 & \multicolumn{4}{c}{DP} & \multicolumn{4}{c}{EO} \\
\cmidrule(lr){2-5}\cmidrule(lr){6-9}
 & Acc.\,$\uparrow$ & $\mathcal{U}(h)\downarrow$ & Deploy gap\,$\downarrow$ & $|$slope$|\downarrow$
 & Acc.\,$\uparrow$ & $\mathcal{U}(h)\downarrow$ & Deploy gap\,$\downarrow$ & $|$slope$|\downarrow$ \\
\midrule
ERM              & $0.751${\scriptsize$\pm.005$} & $0.122${\scriptsize$\pm.019$} & $0.032${\scriptsize$\pm.007$} & $0.041${\scriptsize$\pm.021$}
                 & $0.751${\scriptsize$\pm.005$} & $0.081${\scriptsize$\pm.022$} & $0.060${\scriptsize$\pm.011$} & $0.025${\scriptsize$\pm.021$} \\
STABLE           & $0.764${\scriptsize$\pm.006$} & $0.052${\scriptsize$\pm.003$} & $0.015${\scriptsize$\pm.002$} & $0.067${\scriptsize$\pm.019$}
                 & $0.765${\scriptsize$\pm.004$} & $0.022${\scriptsize$\pm.002$} & $0.030${\scriptsize$\pm.002$} & $0.059${\scriptsize$\pm.006$} \\
ERM + constraint & $0.760${\scriptsize$\pm.007$} & $0.079${\scriptsize$\pm.004$} & $0.055${\scriptsize$\pm.007$} & $0.146${\scriptsize$\pm.022$}
                 & $0.763${\scriptsize$\pm.007$} & $0.036${\scriptsize$\pm.008$} & $0.142${\scriptsize$\pm.015$} & $0.124${\scriptsize$\pm.027$} \\
PostProc         & $0.753${\scriptsize$\pm.004$} & $0.040${\scriptsize$\pm.021$} & $0.069${\scriptsize$\pm.008$} & $0.217${\scriptsize$\pm.027$}
                 & $0.754${\scriptsize$\pm.005$} & $0.057${\scriptsize$\pm.023$} & $0.112${\scriptsize$\pm.019$} & $0.105${\scriptsize$\pm.078$} \\
STABLE$_{\text{DRO}}$ & $0.758${\scriptsize$\pm.007$} & $0.048${\scriptsize$\pm.006$} & $0.036${\scriptsize$\pm.007$} & $0.067${\scriptsize$\pm.017$}
                 & $0.765${\scriptsize$\pm.004$} & $0.022${\scriptsize$\pm.001$} & $0.042${\scriptsize$\pm.010$} & $0.066${\scriptsize$\pm.004$} \\
Shifty           & $0.651${\scriptsize$\pm.006$} & $0.006${\scriptsize$\pm.002$} & $0.403${\scriptsize$\pm.010$} & $0.042${\scriptsize$\pm.017$}
                 & $0.640${\scriptsize$\pm.001$} & $0.003${\scriptsize$\pm.001$} & $0.605${\scriptsize$\pm.005$} & $0.011${\scriptsize$\pm.009$} \\
\bottomrule
\end{tabular}}
\end{table}
\subsection{The effect of $t$}
\label{app:t-study}
%We denote by $t_{\text{ERM}}$ the target level estimated from ERM and used throughout our experiments. In addition, we consider a grid of 25 values for $t$ in $(0.05, 0.95)$. We train STABLE under both DP and EO on Adult dataset for all the values of $t$, using the hyperparameters tuned at $t_{\text{ERM}}$, over 10 random seeds. For DP, Figure~\ref{fig:t-study-adult-DP} shows that the accuracy is fairly stable over the first half of the grid, then it slightly peaks around $t=0.5$, then dramatically drops, suggesting that the problem no longer converges. By contrast, the unfairness and the deployment gap remain very stable. For EO, Figure~\ref{fig:t-study-adult-EO} shows that the accuracy increases as $t$ approaches $t_{\text{ERM}}$, remains stable in its neighbourhood, then drops dramatically as  $t$ moves away, again indicating non-convergence. The unfairness and deployment gap increase as $t$ approaches $t_{\text{ERM}}$, then slightly peak, and finally decrease gradually.
We denote by $t_{\mathrm{ERM}}$ the target level estimated from ERM and used throughout our experiments. In addition, we consider a grid of 25 values for $t$ in $[0.05, 0.95]$. We train STABLE under both DP and EO on Adult dataset for all the values of $t$, using the hyperparameters tuned at $t_{\mathrm{ERM}}$, over 10 random seeds. Figure~\ref{fig:t_study} shows that, for both constraints, the unfairness and the deployment gap remain low over the whole grid, well below those of ERM; under EO, they are slightly larger around $t_{\mathrm{ERM}}$. By contrast, the accuracy depends on $t$. It is highest for $t$ close to $t_{\mathrm{ERM}}$, where it is close to that of ERM, and decreases as $t$ moves away from it. For extreme values of $t$, it drops to the accuracy of a constant classifier.%: reaching a level far from the one of the data comes at the cost of accuracy. Under EO, the accuracy is even lower for $t$ between $0.3$ and $0.45$.
\iffalse
\begin{figure}[t]
    \centering
    \begin{subfigure}{\linewidth}
        \centering
        \includegraphics[width=\linewidth]{figures/t_study_adult_DP.pdf}
        \caption{The effect of target $t$ on Accuracy, Unfairness, and Deployment Gap for STABLE under DP. Shaded regions show $\pm$ std over 20 random seeds.}
        \label{fig:t-study-adult-DP}
    \end{subfigure}

    \vspace{0.5em}

    \begin{subfigure}{\linewidth}
        \centering
        \includegraphics[width=\linewidth]{figures/t_study_adult_EO.pdf}
        \caption{The effect of target $t$ on Accuracy, Unfairness, and Deployment Gap for STABLE under EO. Shaded regions show $\pm$ std over 20 random seeds.}
        \label{fig:t-study-adult-EO}
    \end{subfigure}
    \caption{The effect of target $t$.}
\end{figure}
\fi

\begin{figure}[t]
  \centering
  \includegraphics[width=\linewidth]{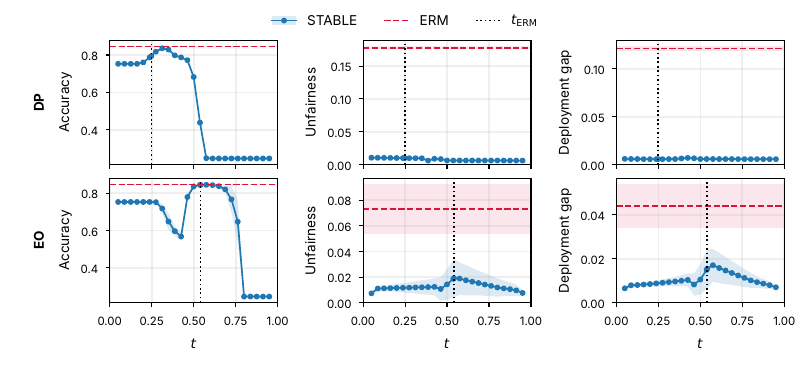}
  \caption{The effect of the target $t$ on STABLE (Adult; top: DP, bottom: EO). Dotted line: $t_{\mathrm{ERM}}$. Shaded regions show $\pm$ std over 10 random seeds.}
  \label{fig:t_study}
\end{figure}

% ------------ neurips checklist ----------
\newpage
\clearpage
\section*{NeurIPS Paper Checklist}

\begin{enumerate}

\item {\bf Claims}
    \item[] Question: Do the main claims made in the abstract and introduction accurately reflect the paper's contributions and scope?
    \item[] Answer: \answerYes{} % Replace by \answerYes{}, \answerNo{}, or \answerNA{}.
    \item[] Justification: We introduce a stability-based view of fairness in Definition~\ref{def:stability}. We show its equivalence to classical notions of fairness in Propositions~\ref{prop:indiv-worst-stab},\ref{prop:DP-stab},\ref{prop:EOpp-stab}. We provide generalization bounds in Theorem~\ref{thm:main-rates}. We illustrate numerical examples in Section~\ref{sec:exp}. Further examples and discussions are included in the Supplementary Material.

\item {\bf Limitations}
    \item[] Question: Does the paper discuss the limitations of the work performed by the authors?
    \item[] Answer: \answerYes{} % Replace by \answerYes{}, \answerNo{}, or \answerNA{}.
    \item[] Justification: Limitations are discussed in the Conclusion.

\item {\bf Theory assumptions and proofs}
    \item[] Question: For each theoretical result, does the paper provide the full set of assumptions and a complete (and correct) proof?
    \item[] Answer: \answerYes{} % Replace by \answerYes{}, \answerNo{}, or \answerNA{}.
    \item[] Justification: All the assumptions are presented and discussed in the main body (Assumptions~\ref{ass:loss:bd}, \ref{ass:reg-rate}, \ref{ass:feas}, \ref{ass:bound-p}, \ref{ass:bound-pi}). Due to space limitations, all the proofs are deferred to the Supplementary Material.

    \item {\bf Experimental result reproducibility}
    \item[] Question: Does the paper fully disclose all the information needed to reproduce the main experimental results of the paper to the extent that it affects the main claims and/or conclusions of the paper (regardless of whether the code and data are provided or not)?
    \item[] Answer: \answerYes{} % Replace by \answerYes{}, \answerNo{}, or \answerNA{}.
    \item[] Justification: All the experiments are reproducible. The experiments validate the claims of the paper.

\item {\bf Open access to data and code}
    \item[] Question: Does the paper provide open access to the data and code, with sufficient instructions to faithfully reproduce the main experimental results, as described in supplemental material?
    \item[] Answer: \answerYes{} % Replace by \answerYes{}, \answerNo{}, or \answerNA{}.
    \item[] Justification: The considered datasets are benchmark datasets. The source code link is provided in Section~\ref{sec:exp}.

\item {\bf Experimental setting/details}
    \item[] Question: Does the paper specify all the training and test details (e.g., data splits, hyperparameters, how they were chosen, type of optimizer) necessary to understand the results?
    \item[] Answer: \answerYes{} % Replace by \answerYes{}, \answerNo{}, or \answerNA{}.
    \item[] Justification: All the information regarding the training and the evaluation is provided in Section~\ref{sec:exp}, with further discussion and details in Appendix~\ref{app:experiments}.

\item {\bf Experiment statistical significance}
    \item[] Question: Does the paper report error bars suitably and correctly defined or other appropriate information about the statistical significance of the experiments?
    \item[] Answer: \answerYes{} % Replace by \answerYes{}, \answerNo{}, or \answerNA{}.
    \item[] Justification: All the reported results are evaluated on 10 independent random seeds for Adult and Compas datasets, and 5 independent random seeds for ACS. The mean results with the respected variances, are presented with all results. Additional discussion is included in Appendix~\ref{app:experiments}.

\item {\bf Experiments compute resources}
    \item[] Question: For each experiment, does the paper provide sufficient information on the computer resources (type of compute workers, memory, time of execution) needed to reproduce the experiments?
    \item[] Answer: \answerYes{} % Replace by \answerYes{}, \answerNo{}, or \answerNA{}.
    \item[] Justification: The information regarding computer resources is provided in the beginning of Appendix~\ref{app:experiments}.
    
\item {\bf Code of ethics}
    \item[] Question: Does the research conducted in the paper conform, in every respect, with the NeurIPS Code of Ethics \url{https://neurips.cc/public/EthicsGuidelines}?
    \item[] Answer: \answerYes{} % Replace by \answerYes{}, \answerNo{}, or \answerNA{}.
    \item[] Justification: The paper conforms to the NeurIPS Code of Ethics.

\item {\bf Broader impacts}
    \item[] Question: Does the paper discuss both potential positive societal impacts and negative societal impacts of the work performed?
    \item[] Answer: \answerYes{} % Replace by \answerYes{}, \answerNo{}, or \answerNA{}.
    \item[] Justification: Since the paper studies algorithmic fairness, its broader positive societal motivation is inherent to the problem setting and is stated in the introduction, as well as in Appendix~\ref{app:disc}.
    
\item {\bf Safeguards}
    \item[] Question: Does the paper describe safeguards that have been put in place for responsible release of data or models that have a high risk for misuse (e.g., pre-trained language models, image generators, or scraped datasets)?
    \item[] Answer: \answerNA{}
    % Replace by \answerYes{}, \answerNo{}, or \answerNA{}.
    \item[] Justification: The paper does not introduce high-risk models or datasets of this kind, and the experiments use standard benchmark datasets.

\item {\bf Licenses for existing assets}
    \item[] Question: Are the creators or original owners of assets (e.g., code, data, models), used in the paper, properly credited and are the license and terms of use explicitly mentioned and properly respected?
    \item[] Answer: \answerYes{}
    % Replace by \answerYes{}, \answerNo{}, or \answerNA{}.
    \item[] Justification: All datasets and benchmark methods used in the paper are referenced, and their use follows standard research practices.

\item {\bf New assets}
    \item[] Question: Are new assets introduced in the paper well documented and is the documentation provided alongside the assets?
    \item[] Answer: \answerYes{}
    % Replace by \answerYes{}, \answerNo{}, or \answerNA{}.
    \item[] Justification: We introduce a novel definition of fairness based on distributional stability, together with the associated methodology, theoretical results, and numerical results. These are documented in the main text and supplementary material.

\item {\bf Crowdsourcing and research with human subjects}
    \item[] Question: For crowdsourcing experiments and research with human subjects, does the paper include the full text of instructions given to participants and screenshots, if applicable, as well as details about compensation (if any)? 
    \item[] Answer: \answerNA{} % Replace by \answerYes{}, \answerNo{}, or \answerNA{}.
    \item[] Justification: The paper does not involve crowdsourcing nor research with human subjects.

\item {\bf Institutional review board (IRB) approvals or equivalent for research with human subjects}
    \item[] Question: Does the paper describe potential risks incurred by study participants, whether such risks were disclosed to the subjects, and whether Institutional Review Board (IRB) approvals (or an equivalent approval/review based on the requirements of your country or institution) were obtained?
    \item[] Answer: \answerNA{} % Replace by \answerYes{}, \answerNo{}, or \answerNA{}.
    \item[] Justification: The paper does not involve crowdsourcing nor research with human subjects.

\item {\bf Declaration of LLM usage}
    \item[] Question: Does the paper describe the usage of LLMs if it is an important, original, or non-standard component of the core methods in this research? Note that if the LLM is used only for writing, editing, or formatting purposes and does \emph{not} impact the core methodology, scientific rigor, or originality of the research, declaration is not required.
    %this research? 
    \item[] Answer: \answerNA{}
    % Replace by \answerYes{}, \answerNo{}, or \answerNA{}.
    \item[] Justification: The core method development in this research does not involve LLMs as any important, original, or non-standard components.

\end{enumerate}

\end{document}